\documentclass[10pt]{article}

\usepackage[english]{babel}

\usepackage[letterpaper, textwidth=6in, textheight=8in, centering]{geometry}

\usepackage{float}
\usepackage{amsmath}
\usepackage{mathrsfs}
\usepackage{makecell}
\usepackage{booktabs}
\usepackage{setspace}
\usepackage{bm}
\usepackage{tikz}
\usepackage{accents,amssymb,amsthm, color,enumerate,fancyhdr, geometry,hyperref}
\usepackage{empheq}
\usepackage[all,pdf]{xy}
\usepackage{pifont}
\usepackage[perpage,symbol*]{footmisc}
\usepackage{CJKutf8}
\usepackage{cases}
\usepackage{verbatim}
\usepackage{graphicx}
\usepackage[]{hyperref}
\usepackage{subfigure}

\usepackage{graphicx}
\usepackage[sort&compress]{natbib}
\usepackage{doi}
\usepackage{caption}
\usepackage{graphicx}
\usepackage{float}
\usepackage{amsfonts} 
\usepackage{nicematrix}
\usepackage{nicefrac}       
\usepackage{microtype}      
\usepackage{lipsum}		

\newtheorem{theorem}{Theorem}[section]
\newtheorem{lemma}{Lemma}[section]

\newtheorem{definition}{Definition}[section]
\usepackage{graphicx}
\usepackage{algorithm} 
\usepackage{algorithmic} 
\usepackage{multirow} 
\usepackage{amsmath} 
\usepackage{xcolor}
\usepackage{multirow}
\usepackage[normalem]{ulem} 
\usepackage[mathlines]{lineno}

\begin{document}
\title{\bf Complex Diffusion Maps with \texorpdfstring{$\omega$}{w}-Parameterized Kernels Revealing Inherent Harmonic Representations}

\author
{Tongzhen Dang
\thanks{ Institute of Science and Technology for Brain-Inspired Intelligence, Fudan University, Shanghai,  China.  (\href{mailto:tzdang21@m.fudan.edu.cn}{tzdang21@m.fudan.edu.cn})} 
\and Weiyang Ding
\thanks{Corresponding author. Institute of Science and Technology for Brain-Inspired Intelligence, Fudan University, Shanghai, China; Key Laboratory of Computational Neuroscience and Brain-Inspired Intelligence (Fudan University), Ministry of Education, China. (\href{mailto: dingwy@fudan.edu.cn}{dingwy@fudan.edu.cn})}
\and Michael K. Ng
\thanks{Department of Mathematics, Hong Kong Baptist University, Hong Kong, China. (\href{mailto:michael-ng@hkbu.edu.hk}{michael-ng@hkbu.edu.hk})}}

\date{Received: \today}

\maketitle
\begin{abstract}
In this paper, we propose Complex Diffusion Maps (CDM), a novel diffusion mapping framework that aims to reveal the dominant complex harmonics of high-dimensional data. Inspired by the local Gaussian kernel relevant to the heat equation and the nonlocal Schrödinger kernel relevant to the Schrödinger equation, we propose a unified family of $\omega$-parameterized complex-valued kernels for the trade-off between local and nonlocal connections. We establish the theoretical foundation based on the operator spectrum theory, where the corresponding diffusion operator, diffusion distance, and complex harmonic maps are well-defined. An optimization-based interpretation of the maps is also developed, aiming to preserve angular structure in the complex diffusion space rather than relying solely on real-valued magnitude. We extensively evaluate CDM on both synthetic and real-world datasets. The complex-valued kernel amplifies differences among easily confusable samples, improving discriminative power over both linear and nonlinear methods based on real-valued kernels. CDM remains robust in high-noise settings, yielding a clearer eigengap that enhances spectral separation. For resting-state fMRI data, CDM captures more strongly correlated and nonlocal spatiotemporal dynamics. Without task-specific tuning, CDM achieves competitive performance on a public EEG sleep dataset, while maintaining high computational efficiency compared with both traditional machine learning and deep neural network approaches, highlighting its generality and practical value.


\end{abstract}
\section{Introduction}
Exploring the underlying feature patterns of high-dimensional data has always been an important topic in data mining and pattern recognition. Although classical linear methods like PCA~\cite{wold1987principal} and Multidimensional Scaling (MDS)~\cite{davison2000multidimensional} are widely used, they struggle to handle complicated data with underlying nonlinearity. Kernel PCA~\cite{scholkopf1997kernel} extends PCA via the kernel trick to capture nonlinear variations, but it does not explicitly account for the intrinsic manifold geometry of the data. 

Motivated by this limitation, manifold learning methods have been developed, evolving from distance-preserving embeddings to graph-based geometric modeling, and further to probabilistic diffusion and visualization techniques. Early approaches, including Isometric Mapping (ISOMAP)~\cite{tennenbaum2000global} and Locally Linear Embedding (LLE)~\cite{roweis2000nonlinear}, preserve global or local geometric structures. Subsequently, Laplacian eigenmaps~\cite{belkin2003laplacian} and Diffusion Maps (DM)~\cite{coifman2005geometric} construct graph-based operators to characterize the underlying manifold geometry. More recent techniques, including t-Distributed Stochastic Neighbor Embedding (t-SNE)~\cite{van2008visualizing}, adopt a probabilistic formulation to emphasize neighborhood preservation for visualization.

Among these approaches, Diffusion Maps (DM)~\cite{coifman2005geometric} is a powerful method that embeds high-dimensional data into a low-dimensional manifold while preserving the intrinsic geometry encoded by diffusion distances. DM constructs a locally connected affinity matrix based on the Gaussian kernel. It is a widely used nonlinear kernel due to its good theoretical background. It is the fundamental solution to the heat equation~\cite{evans2022partial} and approximates the Laplace–Beltrami operator~\cite{rosenberg1997laplacian} on the manifold, capturing nonlinear structures in the data. However, the Gaussian kernel is limited to real-valued information, potentially overlooking richer phase and magnitude information that resides in a latent complex space. Moreover, its reliance on local similarities may mislead the diffusion process in overlapping sets or high-level noise, resulting in incorrect mappings. 

Although complex-valued variants of the Gaussian kernel have been proposed~\cite{bouboulis2010complex, boloix2015complex}, they are primarily designed for complex-valued signals, while the similarity measure itself remains real-valued. In contrast, genuinely complex-valued formulations have been explored in several fields~\cite{bottcher2024complex}. In quantum physics, real-valued formulations are insufficient to capture physical reality, as processes such as electron transport and magnetic interactions, inherently involve complex-valued representations~\cite{renou2021quantum, chen2022ruling, lieb1993fluxes, fanuel2018magnetic}. In machine learning, complex weights have been found to improve neural network performance in terms of accuracy and convergence~\cite{zhang2021magnet, zhang2021optical, lee2022complex}. Similarly, studies in neuroscience have explored the role of complex-valued dynamics in modeling biological neural activity~\cite{kuznetsov1998elements, gomez2016signal, deco2025complex}. 


Motivated by these developments, Deco et al.~\cite{deco2025complex} introduced Complex Harmonics decomposition (CHARM), a novel complex harmonic framework based on a complex-valued kernel derived from the Schrödinger equation~\cite{evans2022partial}. This formulation is closely related to a broader class of Schrödinger-type models that naturally encode nonlocal interactions. In quantum and statistical physics, fractional Schrödinger operators and models with power-law interaction kernels have been used to describe long-range couplings, where interactions remain non-vanishing even across large distances~\cite{laskin2002fractional, tarasov2006fractional}. Similarly, in the theory of nonlinear PDEs, long-range Schrödinger equations capture effects that decay slowly over time or space, thereby preserving nonlocal influence~\cite{strauss1990nonlinear}.

Along this line, Deco et al. showed that discretizing the linear Schrödinger equation naturally yields a nonlocal propagator kernel~\cite{deco2025complex}. Unlike traditional Gaussian kernels, which enforce locality by rapidly vanishing weights, the Schrödinger kernel assigns complex phase factors that vary with distance while maintaining nonzero coupling between distant nodes. This allows the Schrödinger kernel to be naturally interpreted as implementing nonlocal connections in the system, providing a powerful tool to capture the distributed and coherent nature of brain dynamics~\cite{deco2021rare,deco2025non}. However, CHARM constructs its diffusion operator by taking the elementwise squared modulus of the complex-valued kernel, effectively discarding the phase information and retaining only the amplitude. As a result, the target embedding produced by CHARM remains in the real-valued space. Moreover, CHARM lacks a theoretical justification for using the squared modulus in the diffusion process, leaving open questions regarding the explanation of optimal embedding. 

\begin{figure}[htbp]
    \centering
    \includegraphics[width=0.95\textwidth]{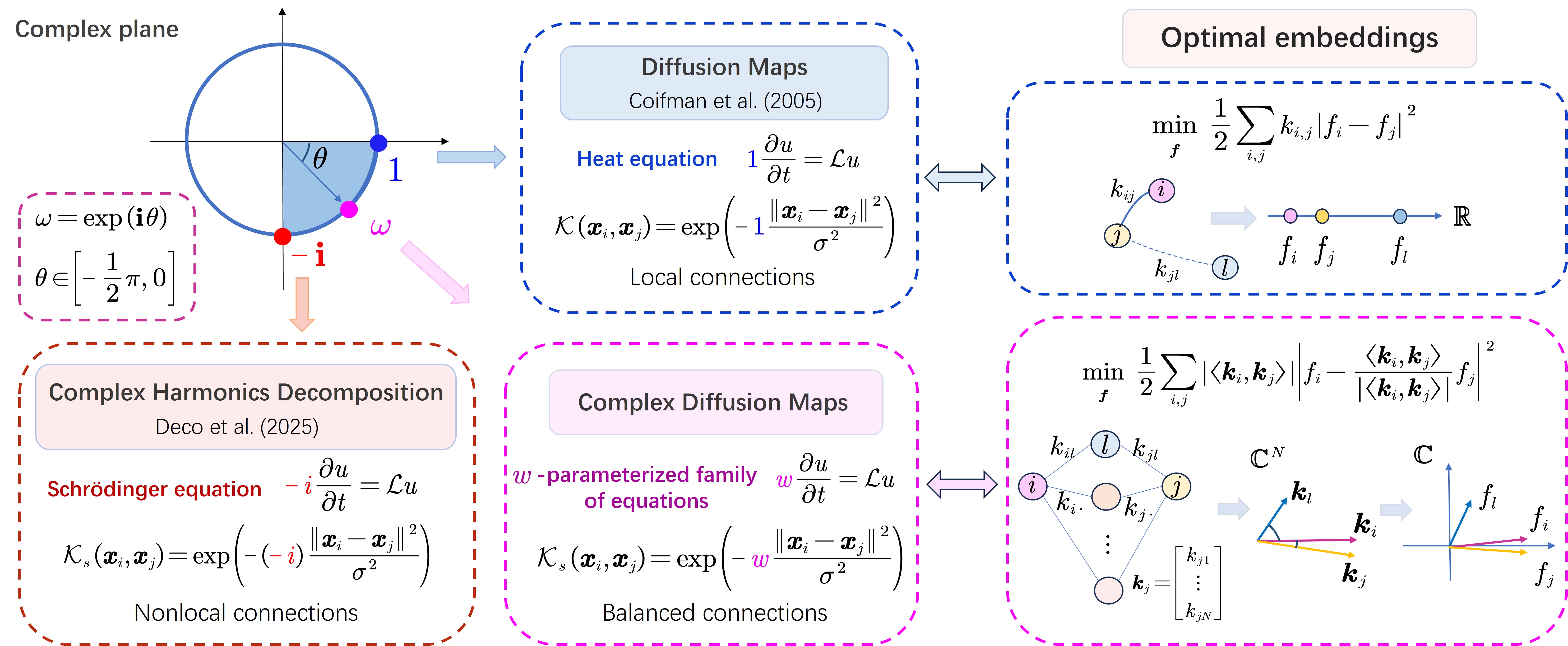}
    \caption{\footnotesize Motivation of the $w$-parameterized complex kernels and the interpretation of optimal embeddings. The Gaussian kernel in Diffusion Maps (DM) is related to the heat equation and mainly captures the local similarity of data. The complex-valued kernel induced by the Schrödinger equation has been shown to encode nonlocal connections through distance-dependent phase, allowing nonzero coupling and thus generating a nonlocal effect between distant points. We propose a generalized family of kernels based on parameter $\omega$, acting as a “dial” controlling the balance between local and nonlocal connections for better adaptation to various data. For optimal embeddings, DM aims to keep locally connected points as close together as possible. Complex Diffusion Maps (CDM) strives to preserve the angle of the global connectivity structure as much as possible in the complex space.}
    \label{complex_kernel}
\end{figure}

The classical DM algorithm constructs an affinity matrix that captures local similarities using the Gaussian kernel induced by the heat equation. The CHARM method leverages the Schrödinger equation to model long-range, nonlocal similarity. However, in a wide range of applications, both local and nonlocal connections play essential roles and should not be neglected. Therefore, we propose Complex Diffusion Maps (CDM), a unified framework that can balance local and nonlocal affinities, supported by comprehensive theoretical analysis and optimization interpretation. Fig.~\ref{complex_kernel} visualizes the motivation for CDM, providing an intuitive overview of the principles designed, naturally leading to the following main contributions of this work.
\begin{itemize}
    \item We propose a family of $\omega$-parameterized complex-valued kernels that unify and generalize the Gaussian kernel reflecting local similarity and the Schrödinger kernel reflecting nonlocal connections to obtain a trade-off via the parameter $\omega$.
    \item We provide a detailed derivation for complex-valued kernels based on operator spectral theory, under which the corresponding diffusion distance and diffusion maps are properly formulated.
    \item We give an optimization-based interpretation of the optimal embeddings, incorporating the angle information of the complex-valued kernel rather than relying solely on its magnitude.
    \item We conduct extensive experiments to verify the effectiveness, robustness and generality of CDM on both synthetic datasets and two real datasets, where CDM exhibits competitive performance compared with machine learning and neural network–based methods.
\end{itemize}

We adopt the following notation conventions in this paper. Scalars are denoted by italic letters such as $x$; vectors by bold lowercase italic letters like $\bm{x}$; matrices by bold uppercase italic letters such as $\bm{A}$; operators by bold upright uppercase letters like $\mathbf{K}$; kernel functions by calligraphic letters such as $\mathcal{K}$; and number fields by blackboard bold letters like $\mathbb{R}$.

This paper is organized as follows. Section~\ref{related work} mainly introduces the classical DM algorithm and its extensions, as well as the Gaussian kernel and Schrödinger kernel induced by physical equations. In Section~\ref{w-para_kernels}, we derive a family of $\omega$-parameterized complex-valued kernels from a set of partial differential equations. Section~\ref{CDM} introduces the CDM algorithm in detail. In Section~\ref{optimal_sec}, we embed the complex harmonics into the corresponding optimization problem and give an explanation of its physical meaning. Finally, Section~\ref{experiments} and~\ref{conclusion} display the experimental results and conclusions, respectively.

\section{Related Work} \label{related work}

\subsection{Diffusion maps}

In this section, we briefly introduce the Diffusion Maps (DM) method~\cite{coifman2005geometric}. Let $\left(\Gamma, \mu \right)$ be a $\sigma$-finite measure space. In many practical applications, $\Gamma$ can be interpreted as a dataset, i.e., $\Gamma \subset \mathbb{R}^M$. Accordingly, $\mu$ denotes the counting measure on $\Gamma$. Define a real-valued symmetric kernel $\mathcal{K}: \Gamma \times \Gamma \to \mathbb{R}$ that quantifies the similarity between two points in $\Gamma$. The symmetrically normalized affinity kernel is defined as 
\begin{displaymath} a\left( \bm{x},\bm{y} \right) =\frac{\mathcal{K} \left( \bm{x},\bm{y} \right)}{\sqrt{v\left( \bm{x} \right)}\sqrt{v\left( \bm{y} \right)}},
\end{displaymath}
where $v(\bm{x})$ is the volume form defined as $v\left( \bm{x} \right) =\int_{\Gamma}{\mathcal{K}\left( \bm{x},\bm{y} \right)  \mathrm{d}\mu \left( \bm{y} \right)}$. Assuming that $v(\bm{x})$ never vanishes and $\mathcal{K}$ is square-integrable and measurable on $(\Gamma,\mu)$, then the diffusion operator $\mathbf{A}:L^2(\Gamma,\mu)\rightarrow L^2(\Gamma,\mu)$ defined as
\begin{equation}\label{eq_DM_operator_A}
\mathbf{A}f(\bm{x})=\int_{\Gamma}a(\bm{x},\bm{y})f(\bm{y})d\mu(\bm{y}),
\end{equation}
is compact and self-adjoint. $\mathbf{A}$ admits the following spectral decomposition~\cite{akhiezer2013theory},
\begin{displaymath}
\mathbf{A}=\sum_{n\in \mathbb{N}}{\lambda _n\left<\cdot,\bm{\phi}_n\right>}\bm{\phi}_n, 
\end{displaymath}
where $\{\lambda_n\}_{n\in\mathbb{N}}$ are the real eigenvalues of $\mathbf{A}$ arranged in non-increasing order, and $\{\bm{\phi}_n\}_{n\in\mathbb{N}}$ are the corresponding orthonormal eigenfunctions. Consequently, 
\begin{displaymath}
a(\bm{x},\bm{y})=\sum_{n\in\mathbb{N}}\lambda_n\bm{\phi}_n(\bm{x})\bm{\phi}_n(\bm{y}).    
\end{displaymath}

For $t \in \mathbb{N}$, the diffusion distance between two points $\boldsymbol{x}$ and $\boldsymbol{y}$ is defined as
\begin{displaymath}
D^t(\bm{x},\bm{y}):=\left\| a^{t}(\bm{x},\cdot )-a^{t}(\bm{y},\cdot ) \right\| _{L^{2}(\Gamma,\mu )}.
\end{displaymath}
where $a^{t}$ denotes the kernel associated with $t$-step diffusion operator $\mathbf{A}^{t}$. Here, $\mathbf{A}^{t}$ denotes the $t$-fold composition of the operator $\mathbf{A}$.

The kernel $a^t(x,\cdot)$ can be interpreted as the probability distribution of reaching
other points from $x$ after $t$ steps~\cite{coifman2005geometric}. 
Therefore, the diffusion distance $D^t(\bm{x},\bm{y})$ measures the discrepancy
between the diffusion distributions starting from $\bm{x}$ and $\bm{y}$. In particular, it is small when $\bm{x}$ and $\bm{y}$ are well connected. 

By substituting the spectral decomposition of the operator $\mathrm{A}$, we obtain
\begin{displaymath}
   D^{t}(\bm{x},\bm{y})=\sqrt{\sum_{n\in \mathbb{N}}
   {\left|\lambda _{n}^{t}\bm{\phi }_n\left( \bm{x} \right) -\lambda _{n}^{t}\bm{\phi }_n\left( \bm{y} \right) \right|^2}}. 
\end{displaymath}
Define the first $s$ maps $\bm{\psi}_s^{t} : \Gamma \rightarrow \mathbb{R}^{s}$ as 
\begin{equation}\label{eq_DM_maps}
    \bm{\psi}_s^{t}(\bm{x})=\left[\begin{array}{llll}\lambda_1^{t} \bm{\phi}_1(\bm{x}), & \lambda_2^{t} \bm{\phi}_2(\bm{x}), & \cdots, & \lambda_s^{t} \bm{\phi}_s(\bm{x})\end{array}\right]^{\top}.
\end{equation}
Thus, 
\begin{equation}\label{eq_DM_dis_eqs_maps}
D^{t}(\bm{x},\bm{y}) \approx \left\| \bm{\psi}_s^{t}(\bm{x})-\bm{\psi}_s^{t}(\bm{y}) \right\| _2.
\end{equation}

Eq.~\eqref{eq_DM_dis_eqs_maps} indicates that DM finds a set of mappings in a low-dimensional space that preserves the diffusion distance in the original manifold space. These maps are defined as $\emph{diffusion maps}$. 

From an optimization perspective, for a discrete dataset $\left\{ \bm{x}_1,\bm{x}_2,\cdots,\bm{x}_N \right\}$, DM aims to solve
\begin{equation}\label{eq_optimal_DM}
   \underset{\bm{f}}{\min}\,\,\frac{1}{2}\sum_{i,j}{k_{ij}}(f_i-f_j)^2,     
\end{equation}
where \( k_{i,j} \coloneqq \mathcal{K}(\bm{x}_i, \bm{x}_j) \) represents the real-valued symmetric weight between samples $\bm{x}_i$ and $\bm{x}_j$, typically defined by the Gaussian kernel. \( \bm{f} \in \mathbb{R}^N \) denotes the embedding function we optimize, subject to a unit-norm constraint~\cite{coifman2005geometric}. Here, we denote \( f_i \coloneqq \bm{f}(\bm{x}_i) \) for simplicity. 

The optimization problem in Eq.~\eqref{eq_optimal_DM} penalizes pairs of nearby points with large differences in their embedding values, encouraging samples with high similarity weights to be mapped close to each other by the function $\bm{f}$.

Let $\bm{K}$ denote the kernel matrix with $\bm{K}_{i,j}=k_{ij}$ and define the diagonal degree matrix $\bm{D}$ by $D_{ii}=\sum_j k_{ij}$. The diffusion operator is then constructed as the normalized diffusion matrix $\boldsymbol{A}$
\begin{displaymath}
\boldsymbol{A}=\boldsymbol{D}^{-1/2}\boldsymbol{K}\boldsymbol{D}^{-1/2}.
\end{displaymath}

The optimal solutions of Eq.~\eqref{eq_optimal_DM} correspond to the leading eigenvectors of the diffusion matrix $\boldsymbol{A}$. These eigenvectors approximate the eigenfunctions $\phi_n$ of the diffusion operator $\mathbf{A}$ defined in Eq.~\eqref{eq_DM_operator_A}~\cite{coifman2005geometric}. Therefore, the optimization variable $\bm{f}$ in Eq.~\eqref{eq_optimal_DM} can be interpreted as a discrete approximation of one eigenfunction $\phi_n$, representing a single coordinate of the diffusion map. By stacking the leading eigenfunctions with their associated diffusion-time eigenvalues $\lambda_n^{t}$, the diffusion map $\bm{\psi}_s^{t}$ in Eq.~\eqref{eq_DM_maps} is obtained, providing a low-dimensional embedding that preserves the diffusion distances between data points.


\subsection{Extensions of Diffusion Maps}\label{extensions_DM}

In this section, we briefly introduce the two classical variants of Diffusion Maps, Vector Diffusion Maps (VDM)~\cite{singer2012vector} and Magnetic Eigenmaps~\cite{fanuel2018magnetic}, aiming to clearly illustrate the connections and differences between our method and these two variants.

One important extension of diffusion maps is the Vector Diffusion Maps (VDM)~\cite{singer2012vector}. VDM was proposed to handle situations where relative rotations among data samples cause their Euclidean distances to overestimate dissimilarity. VDM assigns to each edge an alignment operator that encodes the optimal rotation between neighboring samples, ensuring rotational consistency in the embedding. The corresponding Laplacian is referred to as the \emph{connection Laplacian}, which generalizes the standard graph Laplacian by integrating local geometric transformations. Although this formulation provides a principled way to address misalignment in the data, computing the alignment operators for all pairs of samples can be extremely expensive, as it involves performing two local PCA operations for each sample pair followed by a least-squares computation.

Another related method is the Magnetic eigenmaps~\cite{fanuel2018magnetic}. Magnetic eigenmaps was proposed for embedding directed networks, where standard Laplacians fail to capture asymmetric interactions. The key tool in this approach is the magnetic Laplacian~\cite{shubin1994discrete}, originally derived from the Hamiltonian of a quantum particle under magnetic flux. By attaching a complex phase to each edge, the magnetic Laplacian yields a Hermitian operator whose eigenvectors provide complex-valued embeddings. In this formulation, the symmetric part of the adjacency encodes connection strengths, while the skew-symmetric part is represented through phases, allowing directional relationships to be effectively modeled. This construction has been shown to be particularly useful in applications where directionality is essential~\cite{gomez2021diffusion,he2023diffusion}.

Our method, CDM, is related to both VDM and the Magnetic eigenmaps. VDM addresses rotational misalignments through alignment operators. CDM introduces complex weights directly into the kernel, lifting the real space to the complex domain to couple phase information that may not be captured by the Euclidean distance alone. Moreover, the optimization objective derived in the subsequent section explicitly enforces the preservation of global connectivity in the embedding, rather than relying on local pairwise alignments as in VDM. With respect to magnetic maps, the diffusion operator in our framework is also Hermitian, which makes it formally similar. However, our approach is not motivated by directed graphs or asymmetric interactions but is grounded in the use of complex-valued kernels to enhance the expressive power of diffusion-based embeddings.

Next, we introduce two commonly used kernels and their connections to physical systems, laying the foundation for the generalized complex-valued kernels proposed in the subsequent section.

\subsection{Kernels from physical systems}\label{sec_kernel}

\subsubsection{Heat equation}

The Gaussian kernel is widely used in many fields, such as manifold learning~\cite{belkin2003laplacian}, signal processing~\cite{ito2000gaussian}, and data analysis~\cite{keerthi2003asymptotic}, mainly due to its good mathematical properties and theoretical background. One important theoretical foundation is the heat equation~\cite{evans2022partial}. 

Consider a heat flow on a differentiable Riemannian manifold $\mathcal{M}$. For simplicity, we first consider the one-dimensional case. Let $u(x,t)$ denote the heat distribution at time $t$, where $x\in\mathbb{R}$ denotes a spatial variable. The function  $g:\mathcal{M} \rightarrow \mathbb{R}$ denotes the heat distribution at the initial time. The temperature change at $x$ is related to the average temperature of its surrounding points, so the following equation is used to describe the heat diffusion process,
\begin{displaymath}
\left\{
\begin{aligned}
&\frac{\partial u}{\partial t} = -\mathbf{L} u, \\
&u(x, 0) = g(x),
\end{aligned}
\right.
\end{displaymath}
where $\mathbf{L}$ is the Laplace-Beltrami operator~\cite{rosenberg1997laplacian} of the manifold and $\mathbf{L} u =-\frac{\partial ^2u}{\partial x^2}$. The fundamental solution of the heat equation is given by (see~\cite{evans2022partial})  
\begin{displaymath}
       u\left( x,t \right) =\left( 4\pi t \right) ^{-\frac{1}{2}}\int_{-\infty}^{\infty}{\exp \left( -\frac{\left( x-y \right) ^2}{4t} \right) g\left( y \right) dy}. 
\end{displaymath}
Furthermore, consider
$\mathbf{L} g\left( x \right) =\mathbf{L} u\left( x,0 \right) =-\left. \frac{\partial u}{\partial t} \right|_{t=0}
$, and substitute $u(x,t)$ into it. we obtain, as $t \rightarrow 0$,
\begin{displaymath}
    \mathbf{L}g\left( x \right) \approx \frac{1}{t}\left[ g\left( x \right) -\left( 4\pi t \right) ^{-\frac{1}{2}}\int_{-\infty}^{\infty}{\exp \left( -\frac{\left( x-y \right) ^2}{4t} \right) g\left( y \right) dy} \right].
\end{displaymath}
Using the integral $\int_{-\infty}^{+\infty}{\exp \left( -\frac{y^2}{4t} \right) dy}=\left( 4t\pi \right) ^{\frac{1}{2}}$, we extract the constant coefficient of the first term. Then consider the spatial discretization~\cite{belkin2003laplacian,deco2025complex} and assume that there are $N$ sampling points
$\{x_1,x_2,\dots,x_N\}$. Set $x = x_i$, and $\mathbf{L}$ can be approximated by
\begin{equation}\label{eq_real_lap}
    \mathbf{L}g\left( x_i \right) \approx c\left[ \left( \sum_{j=1}^N{\exp \left( -\frac{\left( x_i-x_j \right) ^2}{4t} \right)} \right) g\left( x_i \right) -\sum_{j=1}^N{\exp \left( -\frac{\left( x_i-x_j \right) ^2}{4t} \right) g\left( x_j \right)} \right] ,
\end{equation}
where the coefficient $c$ is a real number related to $t$. We can take $\exp \left( -\frac{\left( x_i-x_j \right) ^2}{4t} \right)$ as the weight for the approximated discrete Laplacian. The coefficient $c$ does not affect the eigenvectors.

The above derivation is presented for illustration in the one-dimensional case. Replacing the time scale $t$ with a bandwidth parameter
$\sigma$ and extending the formulation to high-dimensional
observations $\bm{x}_i\in\mathbb{R}^d$, the kernel can be generalized as
\begin{displaymath}
    \mathcal{K} (\bm{x}_i,\bm{x}_j)=\exp \left( -\frac{\left\| \bm{x}_i-\bm{x}_j \right\| ^2}{\sigma ^2} \right).
\end{displaymath}
We call the kernel above the \emph{Gaussian kernel} and $\sigma$ the \emph{bandwidth}.

This derivation shows that the Gaussian kernel serves both as the fundamental solution of the heat equation and as an approximation to the Laplace–Beltrami operator~\cite{belkin2006convergence,belkin2008towards}. It focuses on local similarities and can be viewed as a short-range connection. However, nonlocal connections have been shown to be important in neuroscience~\cite{deco2021rare,deco2025complex,deco2025non} and deep learning~\cite{tay2020long}. The next section introduces the recently proposed kernel induced by the Schrödinger equation, which reflects nonlocal similarity.

\subsubsection{Schrödinger equation}

Deco et al.~\cite{deco2025complex,deco2025non} have been inspired by the Schrödinger equation~\cite{griffiths2018introduction} and have developed a complex kernel to capture the nonlocal functional connectivity in human brains. This kernel has been shown to improve information transfer. 

Consider the Schrödinger equation for a free particle in one-dimensional space,
\begin{displaymath}
\left\{
\begin{aligned}
&\mathbf{i}\frac{\partial u_\mathrm{s}}{\partial t}=\mathbf{L}u_\mathrm{s},\\
&u_\mathrm{s}\left( x,0 \right) =g_\mathrm{s}\left( x \right),
\end{aligned}
\right.
\end{displaymath}
where $\mathbf{i}$ is the imaginary unit, and $x\in\mathbb{R}$ denotes the one-dimensional spatial coordinate. $u_\mathrm{s}(x,t)$ and $g_\mathrm{s}$ are complex-valued functions, and $\mathbf{L} u_\mathrm{s} =-\frac{\partial ^2u_\mathrm{s}}{\partial x^2}$. Note that we set the mass of the particle and the Planck constant with ${\hbar}/{2m=1}$for simplicity. 

Comparing the Schrödinger equation with the heat equation, we can see that the former can be obtained simply by replacing the coefficient 1 in the latter with $-\mathbf{i}$. This enables us to follow the same procedure as in the heat equation case. Alternatively, taking the Wick rotation $t \rightarrow \mathbf{i}t$~\cite{zee2010quantum}, we can obtain a similar result to Eq.~\eqref{eq_real_lap} from the Schrödinger perspective
\begin{displaymath}
   \mathbf{L}g_{\mathrm{s}}\left( x_i \right) \approx c_{\mathrm{s}}\left[ \left( \sum_{j=1}^N{\exp \left( \mathbf{i}\frac{\left( x_i-x_j \right) ^2}{4t} \right)} \right) g_{\mathrm{s}}\left( x_i \right) -\sum_{j=1}^N{\exp \left( \mathbf{i}\frac{\left( x_i-x_j \right) ^2}{4t} \right) g_{\mathrm{s}}\left( x_j \right)} \right], 
\end{displaymath}
where the coefficient $c_\mathrm{s}$ is a complex number related to $t$. 

Accordingly, we extend the above formulation from the one-dimensional case to high-dimensional observations $\bm{x}\in\mathbb{R}^d$ as
\begin{equation}\label{eq_schrodinger_kernel}
\mathcal{K} _{\mathrm{s}}(\bm{x}_i,\bm{x}_j)=\exp \left( \mathbf{i}\frac{\left\| \bm{x}_i-\bm{x}_j \right\| ^2}{\sigma ^2} \right),
\end{equation}
We call the kernel the \emph{Schrödinger kernel} and $\sigma$ is a scale parameter.

In~\cite{deco2025complex}, the authors propose the CHARM method, an extension of DM, based on the Schrödinger-induced complex-valued kernel defined in Eq.~\eqref{eq_schrodinger_kernel}.Although the kernel value between $x_i$ and $x_j$ has unit magnitude, their Euclidean distance is encoded in the phase. The interaction does not decay exponentially with distance but instead undergoes pure phase rotation. Therefore, the coupling remains non-vanishing and induces a genuinely nonlocal form of connectivity.

\section{\texorpdfstring{$\omega$}{w}-parameterized Complex Kernels}\label{w-para_kernels}
This section introduces a family of $\omega$-parameterized complex-valued kernels proposed for the trade-off between local and nonlocal connections. This family of complex kernels unifies and generalizes both cases mentioned in Section~\ref{sec_kernel}.

The Gaussian kernel and the Schrödinger kernel introduced above can be viewed as arising from a unified form of governing equations:
\begin{equation}\label{eq_complex_diffusion_eq}
\left\{
\begin{aligned}
&\omega\frac{\partial u_\mathrm{c}}{\partial t}=-\mathbf{L} u_\mathrm{c},\\
&u_\mathrm{c}\left( x,0 \right) =g_\mathrm{c}\left( x \right),  
\end{aligned}
\right.
\end{equation}
where $\omega$ is a complex number with $|\omega|=1$. $u_\mathrm{c}(x,t)$ and $g_\mathrm{c}$ are complex-valued functions and $\mathbf{L} u_\mathrm{c} =-\frac{\partial ^2u_\mathrm{c}}{\partial x^2}$. When $\omega = 1$, it reduces to the heat equation, inducing the Gaussian kernel that emphasizes local similarity. While $\omega = -\mathbf{i}$, it corresponds to the Schrödinger equation and reflects nonlocal connections. If $\omega$ is a complex number on the unit circle from $-\mathbf{i}$ to $1$ as Fig.~\ref{complex_kernel} depicts, that is, $\omega=e^{\mathbf{i}\theta}$ and $\theta\in[-\frac{1}{2}\pi,0]$, then $\omega$ tends to be a “dial” between local and nonlocal connections.

Next, we provide a detailed derivation of the solution to Eq.~\eqref{eq_complex_diffusion_eq}. First, given the form of the Fourier transform used in this paper (another form in which the Fourier transform and its inverse share the same factor of $1/\sqrt{2\pi}$ is also possible),
\begin{displaymath}
    q\left( \lambda \right) =\int_{-\infty}^{\infty}{p\left( \xi \right) e^{-\mathbf{i}\lambda \xi}d\xi},\qquad
p\left( x \right) =\frac{1}{2\pi}\int_{-\infty}^{\infty}{q\left( \lambda \right) e^{\mathbf{i}\lambda x}d\lambda},
\end{displaymath}
where $q(\lambda)$ is the Fourier transform of $p(x)$, denoted by $F\left[ p \right]$; $p(x)$ is the inverse Fourier transform of $q(\lambda)$, denoted by $F^{-1}\left[ q \right]$.

View $t$ in Eq.~\eqref{eq_complex_diffusion_eq} as a parameter, and perform Fourier transform on $x$, denoted by
\begin{displaymath}
F\left[ u_\mathrm{c}\left( x,t \right) \right] =\tilde{u}_\mathrm{c}\left( \lambda ,t \right) ,
\qquad F\left[ g_\mathrm{c}\left( x \right) \right] =\tilde{g}_\mathrm{c}\left( \lambda \right) .
\end{displaymath}
Leveraging the property between the Fourier transform of a function and its derivative $F\left[ p^{\prime}\left( x \right) \right] =\mathbf{i}\lambda F\left[ p\left( x \right) \right] 
$, we can get
\begin{displaymath}
    w\frac{\partial \tilde{u}_\mathrm{c}}{\partial t}=-\lambda ^2\tilde{u}_\mathrm{c}.
\end{displaymath}
For the initial value conditions,
$\tilde{u}_\mathrm{c}\left( \lambda ,0 \right) =\tilde{g}_\mathrm{c}\left( \lambda \right)$,
the above two equations are ordinary differential equations with parameter $\lambda$, and its solution is
\begin{displaymath}
\tilde{u}_\mathrm{c}=\tilde{g}_\mathrm{c}e^{-\lambda ^2\bar{\omega}t}.
\end{displaymath}
Note that the condition $|\omega|=1$ is used here, thus $\frac{1}{\omega}=\bar{\omega}$.

For the Fourier transform of the convolution of two functions, there is an important property $F\left[ p_1*p_2 \right] =F\left[ p_1 \right] \cdot F\left[ p_2 \right] $. Set $p_1=g_\mathrm{c}(x)$ and $F[p_2]=e^{-\lambda ^2\bar{\omega}t}$, then 
\begin{equation}\label{eq_u_c}
    F\left[ p_1*p_2 \right] =\tilde{g}_\mathrm{c}e^{-\lambda ^2\bar{\omega}t}=\tilde{u}_\mathrm{c}.
\end{equation}
Eq.~\eqref{eq_u_c} indicates that the $u_\mathrm{c}$ we require is actually the convolution of $p_1$ and $p_2$, where $p_1$ is known. Now we find $p_2=F^{-1}[e^{-\lambda ^2\bar{\omega}t}]$.
\begin{displaymath}
    \begin{aligned}
    F^{-1}\left[ e^{-\lambda ^2\bar{\omega}t} \right] &=\frac{1}{2\pi}\int_{-\infty}^{\infty}{e^{-\lambda ^2\bar{\omega}t}e^{\mathbf{i}\lambda x}d\lambda}
\\
&=\frac{1}{2\pi}\int_{-\infty}^{\infty}{e^{-\bar{\omega}t\left( \lambda -\mathbf{i}\frac{x}{2t}\omega \right) ^2}e^{-\frac{\omega x^2}{4t}}d\lambda}
\\
&=c_{t,\omega}\exp \left( -\frac{\omega x^2}{4t} \right),
\end{aligned}
\end{displaymath}
where $c_{t,\omega}=\frac{1}{2\pi}\int_{-\infty}^{\infty}{e^{-\bar{\omega}t\mu ^2}d\mu}$ is the result of an integral related to $\omega$ and $t$, and is independent of $x$.
Therefore, the solution of Eq.~\eqref{eq_complex_diffusion_eq} is
\begin{displaymath}
u_{\mathrm{c}}\left( x,t \right) =c_{t,\omega}\int_{-\infty}^{\infty}{\exp \left( -\frac{\omega \left( x-y \right) ^2}{4t} \right) g_{\mathrm{c}}\left( y \right) dy}.
\end{displaymath}

Similar to the process in Section~\ref{sec_kernel}, approximating $\mathbf{L}g_\mathrm{c}(x_i)$ discretely, we can obtain a complex-valued kernel induced by Eq.~\eqref{eq_complex_diffusion_eq} for the high-dimensional version,
\begin{equation}\label{eq_complex_kernel}
    \mathcal{K} _{\mathrm{c}}(\bm{x}_i,\bm{x}_j)=\exp \left( -\omega \frac{\left\| \bm{x}_i-\bm{x}_j \right\| ^2}{\sigma ^2} \right),
\end{equation}
where $\omega=e^{\mathbf{i}\theta},\theta\in[-\frac{1}{2}\pi,0].$ 

So far, we propose the family of $\omega$-parameterized complex-valued kernels, which is a generalization of the Gaussian kernel and the Schrödinger-induced kernel. However, unlike real symmetric or Hermitian matrices, the complex symmetric kernel $\mathcal{K}_\mathrm{c}$ does not possess complete eigendecomposition, making it challenging to embed high-dimensional data into a low-dimensional complex Euclidean space directly. 

Next, we provide the complete theoretical foundation for Complex Diffusion Maps (CDM) based on operator spectral theory to address the diffusion mapping problem for complex symmetric kernels.
\section{Complex Diffusion Maps}\label{CDM}
Let $\left(\Gamma,\mu \right)$ be a $\sigma$-finite measure space, where $\Gamma$ is a set whose elements are abstract entities. In practice, $\Gamma$ often represents a dataset, e.g., $\Gamma \subset \mathbb{R} ^M$. Accordingly, $\mu$ is taken as the counting measure on $\Gamma$. Let $\mathcal{K}:\Gamma \times \Gamma \to \mathbb{C}$ be a complex symmetric kernel satisfying $\mathcal{K}(\boldsymbol{x}, \boldsymbol{y}) = \mathcal{K}(\boldsymbol{y}, \boldsymbol{x})$, which measures the similarity between two points $\boldsymbol{x}, \boldsymbol{y} \in \Gamma$. Consider the integral operator $\mathbf{K}$ induced by the kernel complex symmetric kernel $\mathcal{K}$,
\begin{equation}\label{eq_operator_K}
    \mathbf{K}f\left( \bm{x} \right) =\int_{\Gamma}{\mathcal{K}\left( \bm{x},\bm{y} \right) f\left( \bm{y} \right) \mathrm{d}\mu \left( \bm{y} \right)},
\end{equation}
which acts on the complex square-integrable function space, denoted as $L_{\mathrm{c}}^{2}\left( \Gamma,\mu \right)$. $L_{\mathrm{c}}^{2}\left( \Gamma,\mu \right)$ is a Hilbert space with the inner product:
\begin{displaymath}
\left<f,g\right>=\int_{\Gamma}{f\left( \bm{x} \right) \overline{g\left( \bm{x} \right) }\mathrm{d}\mu \left( \bm{x} \right)}.
\end{displaymath}

Complex symmetric problems often involve a class of operators, $C$-symmetric operators, with similar characteristics to symmetric operators in some aspects. We first introduce the definition of the $C$-operator (Conjugation operator).
\begin{definition}(\cite{garcia2006complex})
Let $\Gamma$ be a separable Hilbert space, $\mathbf{C}$ be an operator on $\Gamma$, if $\mathbf{C}$ satisfies\\
1. Antilinear: $\mathbf{C}\left( \alpha \bm{x}+\beta \bm{y} \right) =\bar{\alpha}\mathbf{C}\left( \bm{x} \right) +\bar{\beta}\mathbf{C}\left( \bm{y} \right) , \forall \alpha ,\beta \in \mathbb{C} , \bm{x},\bm{y}\in \Gamma
$;\\
2. Involutive: $\mathbf{C}^2=\mathbf{I}$;\\
3. Isometric: $\left<\bm{x},\bm{y}\right>=\left<\mathbf{C}\bm{y},\mathbf{C}\bm{x}\right>,\forall\bm{x},\bm{y}\in \Gamma $,\\
then $\mathbf{C}$ is a $C$-operator on $\Gamma$.
\end{definition}

Based on the above definition of the $C$-operator, we further introduce the concept of a $C$-symmetric operator to characterize the structural symmetry an operator exhibits under the action of the associated conjugation operator.

\begin{definition}(\cite{garcia2006complex})
Let $\Gamma$ be a separable Hilbert space, $\mathbf{C}$ be a $C$-operator on $\Gamma$, and $\mathbf{H}$ be a (compact/bounded) linear operator. $\mathbf{H}$ is a (compact/bounded) linear $C$-symmetric operator if $\mathbf{CH}\subset \mathbf{H^*C} $, where $\mathbf{H}^{*}$ is the conjugate operator of $\mathbf{H}$.
\end{definition}
For compact symmetric operators, the eigensystem is complete, and a corresponding spectral decomposition theorem holds. For compact $C$-symmetric operators, the associated singular eigensystem is also complete, and an analogous ``spectral decomposition'' result holds. The following lemma presents this result for compact $C$-symmetric operators for the subsequent derivation.
\begin{lemma}(\cite{garcia2007complex})
\label{lemma_spectrum_decom}
Let $\mathbf{H}$ be a compact $C$-symmetric operator in the Hilbert space $\Gamma$, with the corresponding $C$-operator denoted by $\mathbf{C}$, then $\mathbf{H}$ can be decomposed as
\begin{displaymath}
    \mathbf{H}=\sum_{n\in \mathbb{N}}{\lambda _n\left<\cdot ,\bm{e}_n\right>}\mathbf{C}\bm{e}_n,
\end{displaymath}
where $\left\{ \lambda _n \right\} _{n\in \mathbb{N}}$ are the eigenvalues of $|\mathbf{H}|=\left( \mathbf{H^*H} \right) ^{\frac{1}{2}}$, $|\mathbf{H}|$ is the modulus of the operator $\mathbf{H}$. $\left\{ \bm{e}_n \right\}_{n\in \mathbb{N}}$ is the corresponding orthogonal eigenfunctions.
\end{lemma}

Lemma~\ref{lemma_spectrum_decom} yields the following spectral decomposition of the integral operator $\mathbf{K}$ in CDM, defined in Eq.~\eqref{eq_operator_K}.

\begin{theorem} \label{theorem_K_decomposition}
    Suppose a complex symmetric kernel $\mathcal{K} : \Gamma \times \Gamma \rightarrow \mathbb{C}$ is a Hilbert-Schmidt kernel in $L^{2}_{c}(\Gamma \times \Gamma,\mu \otimes \mu)$, where $\left(\Gamma,\mu \right)$ is a $\sigma$-finite measure space. Then the operator $\mathbf{K}:L^{2}_{c}(\Gamma,\mu) \rightarrow L^{2}_{c}(\Gamma,\mu)$ defined by Eq.~\eqref{eq_operator_K} is a compact operator and has the decomposition as 
    \begin{displaymath}
    \mathbf{K}=\sum_{n\in \mathbb{N}}{\lambda _n\left<\cdot ,\bm{\phi }_n\right>}\overline{\bm{\phi}_n},
    \end{displaymath} 
    where $\left\{ \lambda _n \right\} _{n\in \mathbb{N}}$ are the eigenvalues of $|\mathbf{K}|=\left( \mathbf{K^*K} \right) ^{\frac{1}{2}}$, $|\mathbf{K}|$ is the modulus of the operator $\mathbf{K}$. $\left\{ \bm{\phi}_n \right\}_{n\in \mathbb{N}}$ are the corresponding unit orthogonal eigenfunctions.
\end{theorem}
\begin{proof}
    Since $\mathcal{K}$ is a Hilbert-Schmidt kernel, the corresponding operator $\mathbf{K}$ is a Hilbert-Schmidt operator and thus a compact operator~\cite{conway1994course}. 

    Next, consider an operator $\mathbf{C}$ defined as $\mathbf{C}f(x)=\overline{f\left( x \right) }$. We can show that the operator $\mathbf{C}$ is a $C$-operator on $L_{c}^2(\Gamma,\mu)$. For any $\alpha ,\beta \in \mathbb{C}$, $f,g\in L_{c}^2(\Gamma,\mu)$, $\mathbf{C}\left( \alpha f+\beta g \right)=\overline{\alpha f+\beta g}=\bar{\alpha}\bar{f}+\bar{\beta}\bar{g}=\bar{\alpha}\mathbf{C}f+\bar{\beta}\mathbf{C}g$;  $\mathbf{C}^2f=\mathbf{C}\bar{f}=f$, hence $\mathbf{C}^2=\mathbf{I}$; $\left<f,g\right>=\int_{\Gamma}{f\left( \bm{x} \right) \overline{g\left( \bm{x} \right) }\mathrm{d}\mu \left( \bm{x} \right)}=\int_{\Gamma}{g\left( \bm{x} \right) \overline{f\left( \bm{x} \right) }\mathrm{d}\mu \left( \bm{x} \right)}=\left<\bar{g},\bar{f}\right>=\left<\mathbf{C}g,\mathbf{C}f\right>$.

Then for any $f \in L^{2}_{c}(\Gamma,\mu)$, \begin{displaymath}
\mathbf{CK}f\left( \bm{x} \right) =\mathbf{C}\left( \int_{\Gamma}{\mathcal{K}\left( \bm{x},\bm{y} \right) f\left( \bm{y} \right) \mathrm{d}\mu \left( \bm{y} \right)} \right) =\int_{\Gamma}{\overline{\mathcal{K}\left( \bm{x},\bm{y} \right) }\overline{f\left( \bm{y} \right) }\mathrm{d}\mu \left( \bm{y} \right)},
\end{displaymath}
\begin{displaymath}\mathbf{K^*C}f\left( \bm{x} \right) =\mathbf{K}^*\overline{f\left( \bm{x} \right) }=\int_{\Gamma}{\overline{\mathcal{K}\left( \bm{x},\bm{y} \right) }\overline{f\left( \bm{y} \right) }\mathrm{d}\mu \left( \bm{y} \right)}.
\end{displaymath}
Thus, $\mathbf{CK}\subset \mathbf{K^{*}C}$ and $\mathbf{K}$ is a compact $C$-symmetric operator. From the Lemma~\ref{lemma_spectrum_decom}, $\mathbf{K}$ could be decomposed as \begin{displaymath}
    \mathbf{K}=\sum_{n\in \mathbb{N}}{\lambda _n\left<\cdot ,\bm{\phi}_n\right>}\mathbf{C}\bm{\phi}_n = \sum_{n\in \mathbb{N}}{\lambda _n\left<\cdot ,\bm{\phi }_n\right>}\overline{\bm{\phi}_n},
\end{displaymath}
where $\left\{ \lambda _n \right\} _{n\in \mathbb{N}}$ are the eigenvalues of $|\mathbf{K}|=\left( \mathbf{K^*K} \right) ^{\frac{1}{2}}$, $\left\{ \bm{\phi}_n \right\}_{n\in \mathbb{N}}$ are the corresponding orthogonal eigenfunctions.
\end{proof}

Theorem~\ref{theorem_K_decomposition} indicates that although a complex symmetric kernel may not admit a complete eigendecomposition, it possesses a complete singular eigensystem, which is closely related to its modular operator. Notably, the spectrum of the operator is identical to that of its modulus operator, and their corresponding spectral functions are conjugate under the choice of the conjugation $C$-operator setting. This property preserves the amplitude information, while the phase differs only by a sign. 

Considering the practical implementation, the modular operator usually has no explicit expression except for the prior eigendecomposition of $\mathbf{K^*K}$. Therefore, it is reasonable to consider $\mathbf{K^*K}$ as the diffusion operator in CDM to construct the diffusion process. When performing the embedding, the modular operator should be taken back to ensure that the algorithm naturally reduces to the case of a real symmetric kernel.

\subsection{Diffusion operator}
Based on the original complex kernel $\mathcal{K}$, we define the associated integral operator $\mathbf{K}$. The kernel associated with the composite operator $\mathbf{K^*K}$ can be expressed as
\begin{equation}\label{eq_operator_KK}
    \tilde{\mathcal{K}}\left( \bm{x},\bm{y} \right) \coloneqq \int_{\Gamma}{\overline{\mathcal{K} \left( \bm{x},\bm{z} \right) }\mathcal{K} \left( \bm{y},\bm{z} \right) \mathrm{d}\mu \left( \bm{z} \right)}.
\end{equation}
$\tilde{\mathcal{K}}(\bm{x,y})$ reflects the similarity of the global connections between samples $\bm{x}$ and $\bm{y}$. The conjugate reflects the inner product in the complex space, implying some angle information. Then we define the normalized affinity kernel similar to the DM algorithm as 
\begin{displaymath}
a\left( \bm{x},\bm{y} \right) =\frac{\tilde{\mathcal{K}} \left( \bm{x},\bm{y} \right)}{\sqrt{v\left( \bm{x} \right)}\sqrt{v\left( \bm{y} \right)}},
\end{displaymath}
where $v\left( \bm{x} \right) =\int_{\Gamma}{\left|\tilde{\mathcal{K}}\left( \bm{x},\bm{y} \right)\right|  \mathrm{d}\mu \left( \bm{y} \right)}$. Next, we define the operator 
\begin{equation}\label{eq_operator_A}
    \mathbf{A}f\left( \bm{x} \right) =\int_{\Gamma}{a\left( \bm{x},\bm{y} \right) f\left( \bm{y} \right) \mathrm{d}\mu \left( \bm{y} \right)},
\end{equation}
as the diffusion operator for CDM. The following theorem gives a bound on the spectrum of the complex diffusion operator $\mathbf{A}$.
\begin{theorem}\label{theorem_CDM}
    Assume a complex symmetric kernel $\mathcal{K}$ that satisfies the conditions of Theorem~\ref{theorem_K_decomposition}. Then the kernel $\tilde{\mathcal{K}}(\bm{x,y})$ defined in Eq.~\eqref{eq_operator_KK} is also a Hilbert-Schmidt kernel. Consider the corresponding diffusion operator $\mathbf{A}$ defined in Eq.~\eqref{eq_operator_A}, assuming that $v\left( \bm{x} \right) =\int_{\Gamma}{\left|\tilde{\mathcal{K}}\left( \bm{x},\bm{y} \right)\right|  \mathrm{d}\mu \left( \bm{y} \right)}$ does not vanish on $\Gamma$, then $\mathbf{A}$ is compact and self-adjoint, and the eigenvalues of $\mathbf{A}$ are real and bounded by $[0,1]$.
\end{theorem}
\begin{proof}
    Since $\mathcal{K}$ is a Hilbert-Schmidt kernel, there exists $M>0$ such that 
    \begin{displaymath}
    \iint_{\Gamma \times \Gamma}{\left| \mathcal{K} \left( \bm{x},\bm{y} \right) \right|^2\mathrm{d}\mu \left( \bm{x} \right) \mathrm{d}\mu \left( \bm{y} \right)}<M.
    \end{displaymath}
    For kernel $\tilde{\mathcal{K}}$, apply the Cauchy-Schwarz inequality,
    \begin{displaymath}
    \begin{aligned}
        \left| \tilde{\mathcal{K}}\left( \bm{x},\bm{y} \right) \right|^2&=\left| \int_{\Gamma}{\overline{\mathcal{K} \left( \bm{x},\bm{z} \right) }\mathcal{K} \left( \bm{y},\bm{z} \right) \mathrm{d}\mu \left( \bm{z} \right)} \right|^2\\
        &\leqslant \left( \int_{\Gamma}{\left| \overline{\mathcal{K} \left( \bm{x},\bm{z} \right) } \right|^2\mathrm{d}\mu \left( \bm{z} \right)} \right) \left( \int_{\Gamma}{\left| \mathcal{K} \left( \bm{y},\bm{z} \right) \right|^2\mathrm{d}\mu \left( \bm{z} \right)} \right).
    \end{aligned}
    \end{displaymath}
    Therefore,
    \begin{displaymath}
    \begin{aligned}
        &\iint_{\Gamma \times \Gamma}{\left| \tilde{\mathcal{K}}\left( \bm{x},\bm{y} \right) \right|^2\mathrm{d}\mu \left( \bm{x} \right) \mathrm{d}\mu \left( \bm{y} \right)}\\
        \leqslant &\left( \iint_{\Gamma \times \Gamma}{\left| \mathcal{K} \left( \bm{x},\bm{z} \right) \right|^2\mathrm{d}\mu \left( \bm{x} \right) \mathrm{d}\mu \left( \bm{z} \right)} \right) \left( \iint_{\Gamma \times \Gamma}{\left| \mathcal{K} \left( \bm{y},\bm{z} \right) \right|^2\mathrm{d}\mu \left( \bm{y} \right) \mathrm{d}\mu \left( \bm{z} \right)} \right)<M^2.
    \end{aligned}
    \end{displaymath}
    The kernel $\tilde{\mathcal{K}}$ is also a Hilbert-Schmidt kernel. Moreover, it is easy to verify that $\tilde{\mathcal{K}}(\bm{x,y})=\overline{\tilde{\mathcal{K}}(\bm{y,x})}$. Hence, the operator $\mathbf{A}$ is compact and self-adjoint. We can decompose $\mathbf{A}=\sum_{n\in \mathbb{N}}{\lambda _n\left<\cdot ,\bm{\phi}_n\right>}\bm{\phi}_n$ according to the spectral theorem~\cite{akhiezer2013theory}, where $\left\{\lambda_n\right\}_{n\in \mathbb{N}}$ are the real eigenvalues, and the $\left\{\bm{\phi}_n\right\}_{n\in \mathbb{N}}$ are the orthogonal eigenfunctions.

    Next, we prove that the operator $\mathbf{A}$ is a positive operator. For any $f \in L^{2}_{c}(\Gamma,\mu)$,
    \begin{displaymath}
    \begin{aligned}
        \left< \mathbf{A}f,f \right>& =\iint_{\Gamma \times \Gamma}{\frac{\tilde{\mathcal{K}}\left( \bm{x},\bm{y} \right)}{\sqrt{v\left( \bm{x} \right)}\sqrt{v\left( \bm{y} \right)}}f\left( \bm{y} \right) \overline{f\left( \bm{x} \right) }\mathrm{d}\mu \left( \bm{x} \right) \mathrm{d}\mu \left( \bm{y} \right)}
\\
&=\iint_{\Gamma \times \Gamma}{\int_{\Gamma}{\overline{\mathcal{K} \left( \bm{x},\bm{z} \right) }\mathcal{K} \left( \bm{y},\bm{z} \right) \mathrm{d}\mu \left( \bm{z} \right)}\frac{f\left( \bm{y} \right)}{\sqrt{v\left( \bm{y} \right)}}\frac{\overline{f\left( \bm{x} \right) }}{\sqrt{v\left( \bm{x} \right)}}\mathrm{d}\mu \left( \bm{x} \right) \mathrm{d}\mu \left( \bm{y} \right)}
\\
&=\int_{\Gamma}{\left( \int_{\Gamma}{\mathcal{K} \left( \bm{y},\bm{z} \right) \frac{f\left( \bm{y} \right)}{\sqrt{v\left( \bm{y} \right)}}\mathrm{d}\mu \left( \bm{y} \right)} \right) \left( \int_{\Gamma}{\overline{\mathcal{K} \left( \bm{x},\bm{z} \right) }\frac{\overline{f\left( \bm{x} \right) }}{\sqrt{v\left( \bm{x} \right)}}\mathrm{d}\mu \left( \bm{x} \right)} \right) \mathrm{d}\mu \left( \bm{z} \right)}
\\
&=\left\| g\left( \bm{z} \right) \right\| ^{2}\geqslant 0,
    \end{aligned}
    \end{displaymath}
    where $g\left( \bm{z} \right) \coloneqq \int_{\Gamma}{\mathcal{K} \left( \bm{y},\bm{z} \right) \frac{f\left( \bm{y} \right)}{\sqrt{v\left( \bm{y} \right)}}\mathrm{d}\mu \left( \bm{y} \right)}$. Therefore, $\lambda_n \geqslant 0$. Further, consider the upper bound of the eigenvalue,
    \begin{displaymath}
    \begin{aligned}
    \left| \int_{\Gamma}{\frac{\tilde{\mathcal{K}}\left( \bm{x},\bm{y} \right)}{\sqrt{v\left( \bm{y} \right)}}f\left( \bm{y} \right)}\mathrm{d}\mu \left( \bm{y} \right) \right|&\leqslant \left( \int_{\Gamma}{\left| \tilde{\mathcal{K}}\left( \bm{x},\bm{y} \right) ^{\frac{1}{2}} \right|^2\mathrm{d}\mu \left( \bm{y} \right)} \right) ^{\frac{1}{2}}\left( \int_{\Gamma}{\left| \tilde{\mathcal{K}}\left( \bm{x},\bm{y} \right) ^{\frac{1}{2}}\frac{f\left( \bm{y} \right)}{\sqrt{v\left( \bm{y} \right)}} \right|^2}\mathrm{d}\mu \left( \bm{y} \right) \right) ^{\frac{1}{2}}\\
    &\leqslant \left( \int_{\Gamma}{\left| \tilde{\mathcal{K}}\left( \bm{x},\bm{y} \right) \right|\mathrm{d}\mu \left( \bm{y} \right)} \right) ^{\frac{1}{2}}\left( \int_{\Gamma}{\left| \tilde{\mathcal{K}}\left( \bm{x},\bm{y} \right) \right|\frac{\left| f\left( \bm{y} \right) \right|^2}{v\left( \bm{y} \right)}}\mathrm{d}\mu \left( \bm{y} \right) \right) ^{\frac{1}{2}}\\
	&=\sqrt{v\left( \bm{x} \right)}\left( \int_{\Gamma}{\left| \tilde{\mathcal{K}}\left( \bm{x},\bm{y} \right) \right|\frac{\left| f\left( \bm{y} \right) \right|^2}{v\left( \bm{y} \right)}}\mathrm{d}\mu \left( \bm{y} \right) \right) ^{\frac{1}{2}}.\\
    \end{aligned}
    \end{displaymath}
Therefore,
\begin{displaymath}
\begin{aligned}
	\left< \mathbf{A}f,f \right> &\leqslant \int_{\Gamma}{\frac{1}{\sqrt{v\left( \bm{x} \right)}}\sqrt{v\left( \bm{x} \right)}\left( \int_{\Gamma}{\left| \tilde{\mathcal{K}}\left( \bm{x},\bm{y} \right) \right|\frac{\left| f\left( \bm{y} \right) \right|^2}{v\left( \bm{y} \right)}}\mathrm{d}\mu \left( \bm{y} \right) \right) ^{\frac{1}{2}}\overline{f\left( \bm{x} \right) }\mathrm{d}\mu \left( \bm{x} \right)}\\
	&\leqslant \left( \int_{\Gamma}{\left| \overline{f\left( \bm{x} \right) } \right|^2\mathrm{d}\mu \left( \bm{x} \right)} \right) ^{\frac{1}{2}}\left( \int_{\Gamma}{\int_{\Gamma}{\left| \tilde{\mathcal{K}}\left( \bm{x},\bm{y} \right) \right|\frac{\left| f\left( \bm{y} \right) \right|^2}{v\left( \bm{y} \right)}}\mathrm{d}\mu \left( \bm{y} \right) \mathrm{d}\mu \left( \bm{x} \right)} \right) ^{\frac{1}{2}}\\
	&=\left\| f \right\| \left( \int_{\Gamma}{\left( \int_{\Gamma}{\left| \tilde{\mathcal{K}}\left( \bm{x},\bm{y} \right) \right|\mathrm{d}\mu \left( \bm{x} \right)} \right) \frac{\left| f\left( \bm{y} \right) \right|^2}{v\left( \bm{y} \right)}\mathrm{d}\mu \left( \bm{y} \right)} \right) ^{\frac{1}{2}}\\
	&=\left\| f \right\| \left(\int_{\Gamma}{v\left( \bm{y} \right) \frac{\left| f\left( \bm{y} \right) \right|^2}{v\left( \bm{y} \right)}\mathrm{d}\mu \left( \bm{y} \right)}\right)^{\frac{1}{2}}\\
	&=\left\| f \right\| ^2.\\
\end{aligned}
    \end{displaymath}
    From the minimax theorem of compact self-adjoint operators~\cite{conway1994course}, we know that $\lambda _{\max}\left( \mathbf{A} \right) \leqslant 1$. Thus, the eigenvalues of $\mathbf{A}$ are bounded by $\lambda_n \in [0,1]$.
\end{proof}

\subsection{Diffusion distance and maps}
Next, we define the diffusion distance in the complex-valued setting. From Theorem~\ref{theorem_CDM}, the diffusion kernel $a(\bm{x,y})$ can be written as the spectral decomposition $a(\bm{x,y}) = \sum_{n\in \mathbb{N}}\lambda_n\bm{\phi}_n(\bm{x})\bm{\phi}_n(\bm{y})$ with $\lambda_n \in [0,1]$. 

To ensure that the algorithm reduces to the classical diffusion mapping when taking a real kernel, we consider the kernel $\tilde{a}(\bm{x,y}) = \sum_{n\in \mathbb{N}}\sqrt{\lambda_n}\bm{\phi}_n(\bm{x})\bm{\phi}_n(\bm{y})$ associated with the positive square root of operator $\mathrm{A}$ to define the diffusion distance, where $\mathbf{A}$ is proven to be positive self-adjoint in Theorem~\ref{theorem_CDM}. Note that $\sqrt{\lambda_n}\in [0,1]$ helps ensure stability of the embedding as the number of diffusion steps increases. 

Then, the $t$-step diffusion distance is defined as 
\begin{equation}\label{eq_distance}
D^t(\bm{x},\bm{y}):=\left\| \tilde{a}^t(\bm{x},\cdot )-\tilde{a}^t(\bm{y},\cdot ) \right\| _{L_{c}^{2}(\Gamma ,\mu )}
=\left(\int_{\Gamma}{\left| \tilde{a}^t\left( \bm{x},\bm{z} \right) -\tilde{a}^t\left( \bm{y},\bm{z} \right) \right|^2\mathrm{d}\mu \left( \bm{z} \right)}\right)^{\frac{1}{2}},
\end{equation}
where $\tilde{a}^{t}$ denotes the kernel associated with the operator $(\mathbf{A}^{1/2})^t$, and the related $\mathbf{A}^{t}$ denotes the $t$-fold composition of the diffusion operator $\mathbf{A}$.

The function $\tilde{a}^t(\bm{x},\bm{z})$ reflects the similarity of the global connection between $\bm{x}$ and $\bm{z}$ after the diffusion in the $t$-step. Intuitively, Eq.~\eqref{eq_distance} considers all paths connecting $\bm{x}$ and $\bm{y}$ through points $\bm{z}$ with similar global diffusion connectivity. If there exist many such paths, $\bm{x}$ and $\bm{y}$ are close in the diffusion space. In this sense, $D^t(\bm{x},\bm{y})$ captures the similarity of $\bm{x}$ and $\bm{y}$ not only through direct pairwise interactions but also via the overall network of consistent global connections.

Using the orthogonality and unit properties of eigenfunctions $\left< \bm{\phi }_n,\bm{\phi }_m \right> =0$ for $m\ne n$, and $\left< \bm{\phi }_n,\bm{\phi }_n \right> =1$ for all $n$, substitute $\tilde{a}(\bm{x,y}) = \sum_{n\in \mathbb{N}}\sqrt{\lambda_n}\bm{\phi}_n(\bm{x})\bm{\phi}_n(\bm{y})$ into Eq.~\eqref{eq_distance}, we can get
\begin{displaymath}
D^t(\bm{x},\bm{y})=\sqrt{\sum_{n\in \mathbb{N}}{\left|\lambda _{n}^{t/2} \bm{\phi }_n\left( \bm{x} \right) -\lambda _{n}^{t/2}\bm{\phi }_n\left( \bm{y} \right) \right|^2}}.
\end{displaymath}

Then we define the $\emph{complex diffusion maps}$ as 
\begin{equation}\label{eq_CDM_maps}
\bm{\psi}^{t}_s(\bm{x})=\left[\begin{array}{llll}\lambda_1^{t/2} \bm{\phi}_1(\bm{x}), & \lambda_2^{t/2} \bm{\phi}_2(\bm{x}), & \cdots, & \lambda_s^{t/2} \bm{\phi}_s(\bm{x})\end{array}\right].
\end{equation}
Thus, 
\begin{equation}\label{eq_CDM_dst_eqs_maps}
    D^{t}(\bm{x},\bm{y}) \approx \left\| \bm{\psi}_s^{t}(\bm{x})-\bm{\psi}_s^{t}(\bm{y}) \right\|_2.
\end{equation}

Eq.~\eqref{eq_CDM_dst_eqs_maps} indicates that CDM finds a set of maps $\bm{\psi}^{t}_s : \Gamma \rightarrow \mathbb{C}^{s}$, which embed the dataset into a low dimensional complex space. At this stage, we have derived the complex harmonic representations which can describe the diffusion distance of the original high-dimensional space using the embeddings in the low-dimensional complex Euclidean space. 

\section{Optimal Embeddings}\label{optimal_sec}
In Section~\ref{CDM}, we derived diffusion maps for complex symmetric kernels in the continuous case from the perspective of operator spectral theory. These maps are in one-to-one correspondence with a set of eigenvalues in $\left[0,1\right]$. In practice, we typically access a finite number $N$ of samples $\left\{ \bm{x}_1,\bm{x}_2,\cdots,\bm{x}_N \right\}$. Let $\bm{K}$ denote the discrete kernel matrix with $\bm{K}_{i,j} = \mathcal{K}(\bm{x}_i, \bm{x}_j)$, and $\mathcal{K}$ is the given complex symmetric kernel like Eq.~\eqref{eq_complex_kernel}. Next, we explain the significance of complex maps through a discrete optimization formulation involving $\bm{K}$.

Let $\bm{k}_{i}$ denote the $i$-th column of $\bm{K}=[\bm{k}_1,\bm{k}_2,\cdots, \bm{k}_N]$. Thus, $\bm{k}_{i}$ indicates the connections between sample $\bm{x}_i$ and all other samples. Let $\bm{f} \in \mathbb{C}^N$ denote the target embedding to be optimized, with $f_i \coloneqq \bm{f}(\bm{x}_i)$ for simplification. 

We expect that the optimal embeddings of $\bm{x}_i$ and $\bm{x}_j$ preserve the distribution of their connectivity structure, which contains balanced local and nonlocal links with other samples based on the $\omega$-parameterized kernel, as much as possible. If samples $i$ and $j$ share similar connectivity patterns with all other samples, $|\left< \bm{k}_i,\bm{k}_j \right> |$ tends to be large. Only when $ f_i^* f_j$  is as close to $\left< \bm{k}_i,\bm{k}_j \right>$ as possible, can the objective function approach zero. Such a goal inspires the following optimization problem,
\begin{equation}\label{eq_optimal}
    \underset{\bm{f}}{\min}\,\,\frac{1}{2}\sum_{i,j}{|\left< \bm{k}_i,\bm{k}_j \right> |\left| f_i-\frac{\left< \bm{k}_i,\bm{k}_j \right>}{|\left< \bm{k}_i,\bm{k}_j \right> |}f_j \right|^2},
\end{equation}
where the normalization factor in Eq.~\eqref{eq_optimal} is to eliminate the influence of magnitude. By expanding the objective function, we obtain, for any $\bm{f} \in \mathbb{C}^{N}$,
\begin{displaymath}
    \begin{aligned}
        &\frac{1}{2}\sum_{i,j}{|\left< \bm{k}_i,\bm{k}_j \right> |\left| f_i-\frac{\left< \bm{k}_i,\bm{k}_j \right>}{|\left< \bm{k}_i,\bm{k}_j \right> |}f_j \right|^2}
\\
=&\frac{1}{2}\sum_{i,j}{|\left< \bm{k}_i,\bm{k}_j \right> |\left( f_{i}^{*}f_i+f_{j}^{*}f_j-2\frac{\left< \bm{k}_i,\bm{k}_j \right>}{|\left< \bm{k}_i,\bm{k}_j \right> |}f_{i}^{*}f_j \right)}
\\
=&\sum_{i,j}{|\left< \bm{k}_i,\bm{k}_j \right> |f_{i}^{*}f_i}-\sum_{i,j}{\left< \bm{k}_i,\bm{k}_j \right> f_{i}^{*}f_j}
\\
=&\bm{f}^*\bm{Df}-\bm{f}^*\left( \bm{K}^*\bm{K} \right) \bm{f}
\\
=&\bm{f}^*\bm{Lf}
    \end{aligned},
\end{displaymath}
where $\bm{D}_{ii} = \sum_j{|(\bm{K}^*\bm{K})_{ij}|} = \sum_j{|\left< \bm{k}_i,\bm{k}_j\right>|}$ and $\bm{L} = \bm{D} - \bm{K}^*\bm{K}$. 

It can be observed that the optimization problem in Eq.~\ref{eq_optimal} is essentially equivalent to minimizing a quadratic form of the Hermitian Laplacian induced by $\bm{K^*K}$. Such a minimization seeks a complex-valued function that varies as smoothly as possible over the graph. This ensures that both the magnitude and the phase of the function remain coherent along edges weighted by the complex affinity kernel. Therefore, unlike the Vector Diffusion Maps (VDM)~\cite{singer2012vector} approach mentioned in Section~\ref{extensions_DM}, which aligns pairwise sample connections through local orthogonal transformations, CDM directly optimizes over the entire kernel matrix to obtain globally consistent complex embeddings. 

Applying the Lagrange multiplier method, minimizing $\bm{f^*Lf}$ is actually a generalized eigenvalue problem $\bm{L}\bm{f}=\lambda\bm{Df}$, with the constraint $\bm{f}^*\bm{Df}=1$, which eliminates the arbitrary scaling factor in the embedding. 

In analogy to the DM framework~\cite{coifman2005geometric}, normalization plays a crucial role in ensuring that the resulting operator captures intrinsic geometric relationships independent of sampling density. Therefore, the normalized Hermitian Laplacian $\tilde{\bm{L}}=\bm{I}-\bm{D}^{-\frac{1}{2}}\bm{K^*KD}^{-\frac{1}{2}}$ is often adopted for its symmetry and stability properties. The eigenpairs of $\bm{L}$ and $\tilde{\bm{L}}$ are related as follows: $\lambda$ and $\bm{f}$ solve the generalized eigenvalue problem $\bm{L}\bm{f}=\lambda\bm{Df}$ if and only if $\lambda$ is an eigenvalue of $\tilde{\bm{L}}$ with eigenvector $\bm{g}=\bm{D}^{\frac{1}{2}}\bm{f}$~\cite{chung1997spectral}. Since $\tilde{\bm{L}}$ and $\bm{D}^{-\frac{1}{2}}\bm{K^*K}\bm{D}^{-\frac{1}{2}}$ share the same eigenvectors, the complex diffusion embedding can equivalently be obtained from the leading eigenvectors of the matrix $\bm{D}^{-\frac{1}{2}}\bm{K^*K}\bm{D}^{-\frac{1}{2}}$.

Therefore, the leading eigenvectors of the matrix $\bm{D}^{-\frac{1}{2}}\bm{K^*K}\bm{D}^{-\frac{1}{2}}$ can be interpreted as a discrete approximation of one eigenfunction $\phi_n$ of the diffusion operator $\mathbf{A}$ in Eq.~\eqref{eq_operator_A}, representing a single coordinate function of the complex diffusion embedding. By stacking the leading eigenvectors with their associated diffusion-time eigenvalues, the first $s$ complex embedding coordinates are obtained, which coincide with the complex diffusion maps defined in Eq.~\eqref{eq_CDM_maps}. This construction yields a low-dimensional complex embedding that preserves the diffusion geometry induced by the complex kernel. The corresponding pseudocode is given in Algorithm~\ref{algorithm 1}.

\begin{algorithm}[htpb]
  \caption{Complex Diffusion Maps (CDM).}
  \label{algorithm 1}
  \begin{algorithmic}[1]
    \REQUIRE
        The data set $\left\{ \bm{x}_1,\bm{x}_2,\cdots,\bm{x}_N \right\}$, the parameter of complex kernel $\omega$, bandwidth $\sigma$, the diffusion step $t$, the dimension of the embedded space $s$.
    \ENSURE
        The low-dimensional complex diffusion maps $\bm{\psi}^{t}_{s}$.
    \STATE {Construct the complex symmetric weight matrix $\bm{K}$ with $\bm{K}_{i,j}=\exp{(-\omega\frac{\| \bm{x}_i-\bm{x}_j \| ^2}{\sigma ^2})}$;}
    \STATE {Compute the Hermitain matrix $\bm{K^*K}$ and corresponding degree matrix $\bm{D}$;}
    \STATE {Construct the Markov process as $\bm{A} = \bm{D}^{-\frac{1}{2}} \bm{K}^* \bm{K} \bm{D}^{-\frac{1}{2}}$;}
    \STATE {Compute the first $k$ eigenvalues $\left\{\lambda_i\right\}_{i=1}^{k}$ and eigenvectors $\left\{\bm{\phi}_i\right\}_{i=1}^{k}$ of $\bm{A}$;}
    \STATE {Get the $t$-step maps $\bm{\psi}^{t}_s=\left[\begin{array}{llll}\lambda_1^{t/2} \bm{\phi}_1, & \lambda_2^{t/2} \bm{\phi}_2, & \cdots, & \lambda_s^{t/2} \bm{\phi}_s\end{array}\right]$;}\\
    \RETURN complex diffusion maps $\bm{\psi}^{t}_{s}$.
\end{algorithmic}
\end{algorithm}
Let $d$ denote the feature dimension of each sample and $N$ the total number of samples, where typically $d \ll N$. The overall time complexity of CDM is $\mathcal{O}(N^2 d + N^3)$, and the space complexity is $\mathcal{O}(N^2)$. Such requirements may be computationally expensive for large datasets. For real-valued kernels, sparsity is often induced by applying a threshold $\epsilon$, setting the entries smaller than $\epsilon$ to zero. However, in the case of complex-valued kernels, this approach is not directly applicable since magnitude comparisons are insufficient, and truncation based solely on modulus may discard essential phase information. To address this issue, we instead employ the Nyström-based extension~\cite{williams2000using} to speed up kernel machines for out-of-sample cases. Due to space limitations, the out-of-sample extension is presented in the Appendix~\ref{Out-of-sample}.

\section{Numerical Experiments}\label{experiments}
In this section, we conduct extensive experiments to validate the effectiveness, robustness, and generalization ability of CDM. The experiments in the first three subsections are all conducted using MATLAB R2022a on a laptop equipped with an Intel(R) Core(TM) i7 processor. For the EEG experiment in Section~\ref{exp_eeg}, manifold learning methods are implemented on a CPU-based system running CentOS Linux 7.9, while deep neural network baselines are trained on an NVIDIA GeForce RTX 2080 Ti GPU.

\subsection{Confusable points clustering}
In this experiment, we demonstrate that complex-valued kernels can better distinguish confusable points by leveraging angle distributions in the complex space, compared to conventional linear and nonlinear real-valued kernel methods.
\subsubsection{A three-point toy example}
This example is a synthetic scenario designed to illustrate cases where data points from different classes may overlap, independent of the specific form of the original data. We construct a $3\times 3$ Euclidean distance matrix for three points, as shown in Fig.~\ref{three_point_example}(A). The first point (blue) is close to the second point (yellow), while the third point (red) is farther from both. Classical real-valued kernel methods, such as DM, rely entirely on Euclidean distance to construct similarity matrices. As shown in the middle panel of Fig.~\ref{three_point_example} (A), the block $2\times2$ on the upper left tends to embed the blue and yellow points closer. In contrast, the complex-valued kernel no longer focuses on local similarity but considers a more balanced angle distribution of global connections. 

\begin{figure}[ht]
    \centering
    \includegraphics[width=0.8\textwidth]{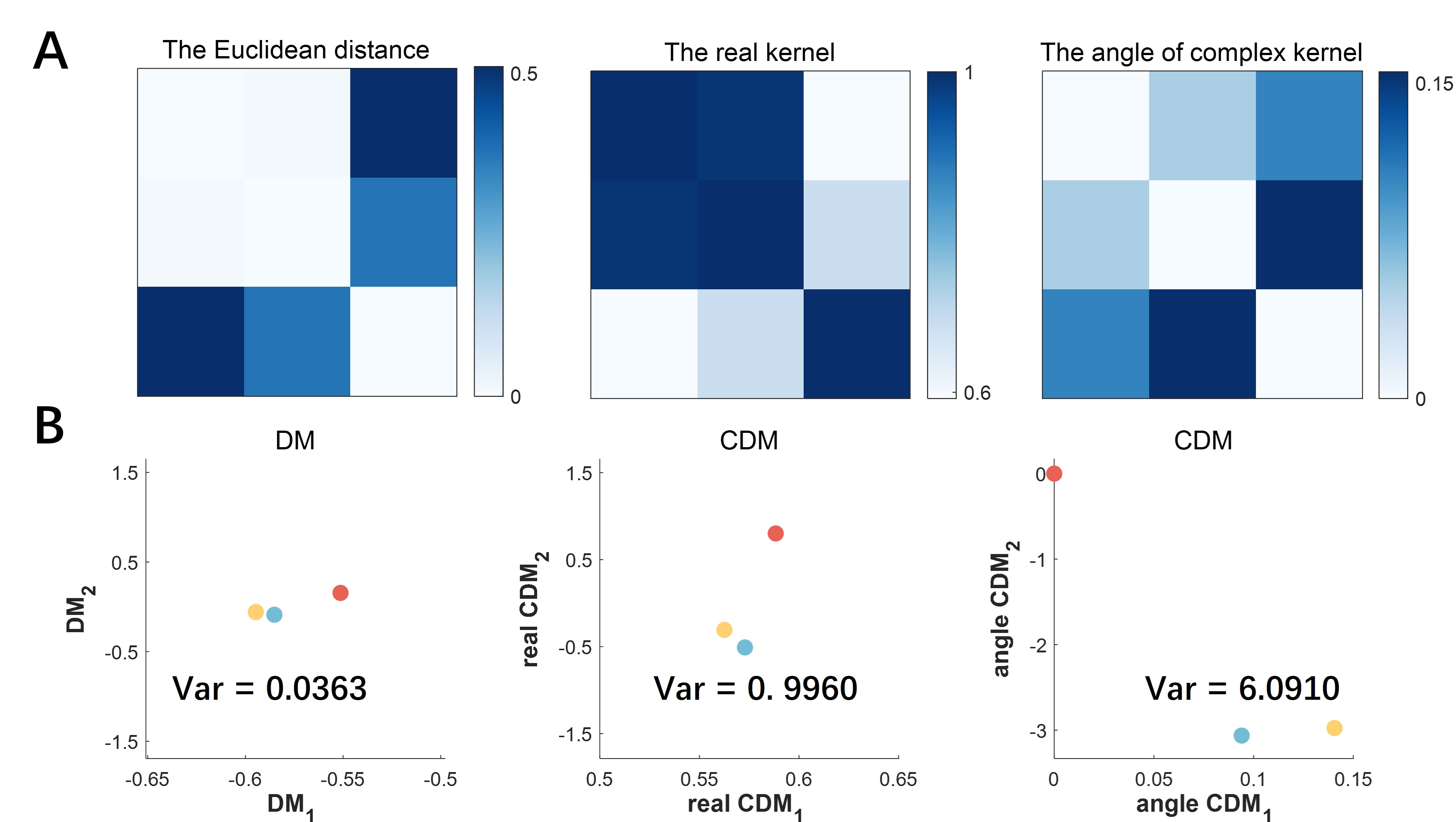}
    \caption{\footnotesize Illustration of the three-point toy example. (A) Euclidean distances (left), real-valued Gaussian kernel (middle), and angles between global connections (right). Points 1 (blue) and 2 (yellow) are close and indistinguishable under the real-valued kernel, while complex weights enhance their differences via angle information. (B) The first two embeddings from DM (left), the real parts of CDM (middle), and the corresponding angles (right). Embedding variances are indicated. Real-valued methods may confuse nearby points, whereas complex-valued kernels leverage phase information to amplify their separation.}
    \label{three_point_example}
\end{figure}

The embeddings of DM and CDM shown in Fig.~\ref{three_point_example}(B) are computed using parameters selected via grid search. The first two embedding dimensions and their corresponding variances show that CDM, which leverages angle information in the complex space, exhibits a global balancing effect on closely spaced points, consistent with the optimization objective in Section~\ref{optimal_sec}. These results suggest that introducing the complex kernel improves the identification of easily confused data points.

\subsubsection{A three-class toy example}\label{exp_artificial}
In order to simulate scenarios where real-valued information is insufficient for class separation, we construct a complex-valued symmetric matrix that explicitly couples amplitude and phase information. Specifically, amplitude and phase are jointly encoded through a weighted combination, and the kernel parameter $\omega$ is chosen to be consistent with this amplitude–phase ratio, ensuring that the induced complex kernel reflects their relative contributions in both magnitude and phase.

Rather than directly generating feature samples, we synthesize a $3\times3$ block-structured matrix that implicitly encodes cluster structures via their statistical properties in the complex domain. Detailed construction procedures are described in Appendix~\ref{Artificial_pointset}.

\begin{figure}[H]
    \centering
    \includegraphics[width=0.9\textwidth]{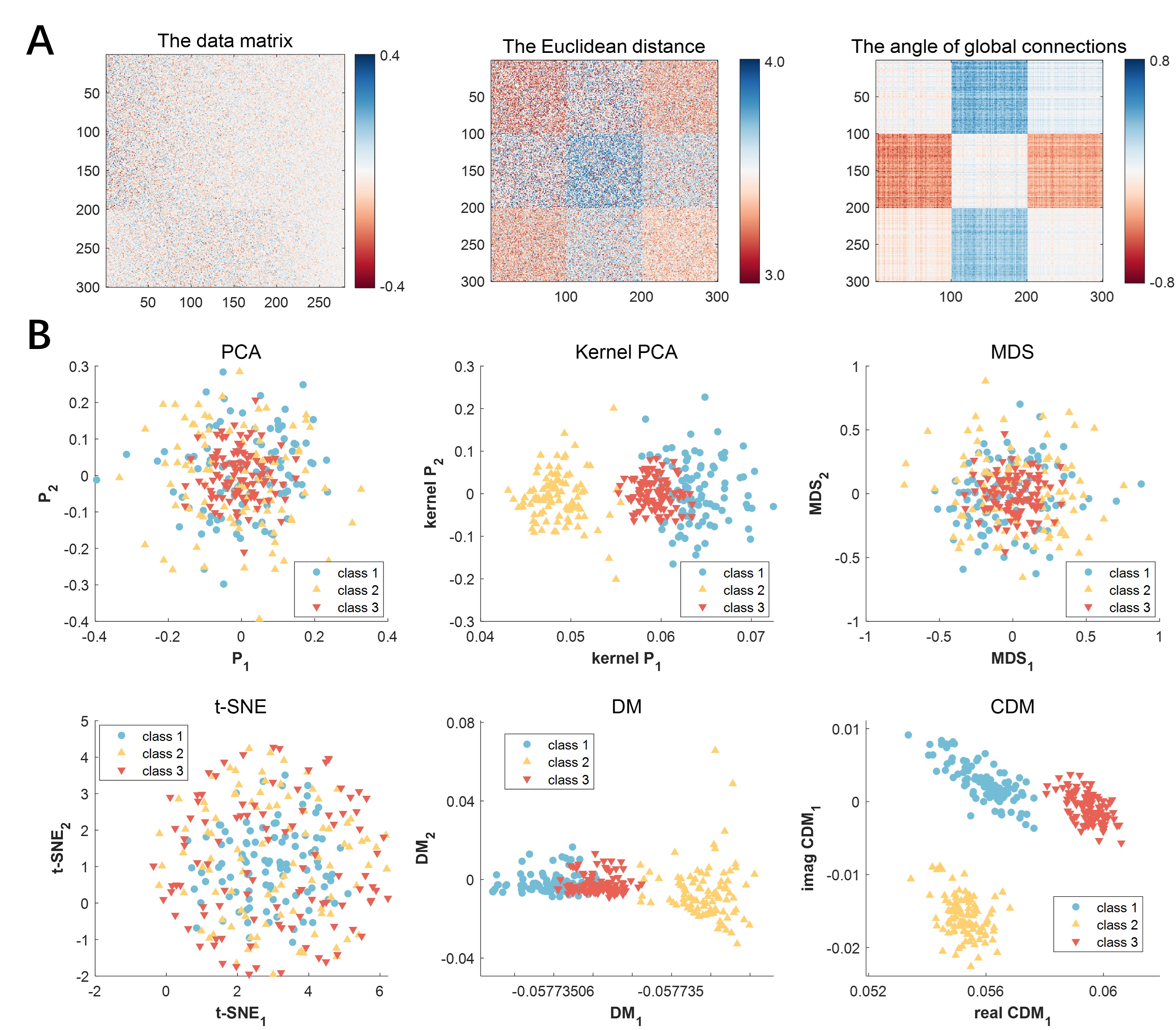}
    \caption{\footnotesize Comparison of CDM with linear and nonlinear real-valued methods on an artificial dataset constructed via amplitude–phase coupling (see Appendix~\ref{Artificial_pointset} for details). (A) Heatmap of the data matrix (left), corresponding Euclidean distances (middle), and angle distribution (right). Classes 1 and 3 exhibit similar distance structures and are difficult to distinguish, while Class 2 shows distinct intra-class distances but similar global connectivity. The angle distribution reveals that Classes 1 and 3 have nearly opposite internal phases, which facilitates their separation. (B) Embedding results of different methods with identical bandwidth settings for kernel-based approaches. Linear methods (PCA, MDS) fail to capture the structure. t-SNE produces a sphere-like representation without clear clusters. Kernel PCA and DM confuse Classes 1 and 3. In contrast, CDM exploits phase information to achieve clearer separation.}
    \label{large_set_example}
\end{figure}

As illustrated in Fig.~\ref{large_set_example}(A), the left panel shows the reconstructed real-valued data derived from the complex domain, where amplitude and phase information are jointly encoded. The middle panel presents the corresponding Euclidean distance matrix, while the right panel visualizes the phase (angle) distribution. 

We compare CDM with both linear methods (PCA, MDS) and nonlinear kernel-based methods (Kernel PCA, t-SNE, and DM), using identical bandwidth settings for all kernel-based approaches. The results are shown in Fig.~\ref{large_set_example}(B).

Linear methods such as PCA and MDS fail to capture the underlying structure of the data. Although t-SNE reveals a spherical structure, it does not provide a meaningful low-dimensional separation. Kernel PCA and DM tend to confuse Classes 1 and 3 due to their similar distance structures. In contrast, CDM exploits phase information to achieve a more discriminative embedding, resulting in clearer separation and improved clustering performance.

\subsection{Noisy signal recognition}
\begin{figure}[ht]
    \centering
    \includegraphics[width=0.9\textwidth]{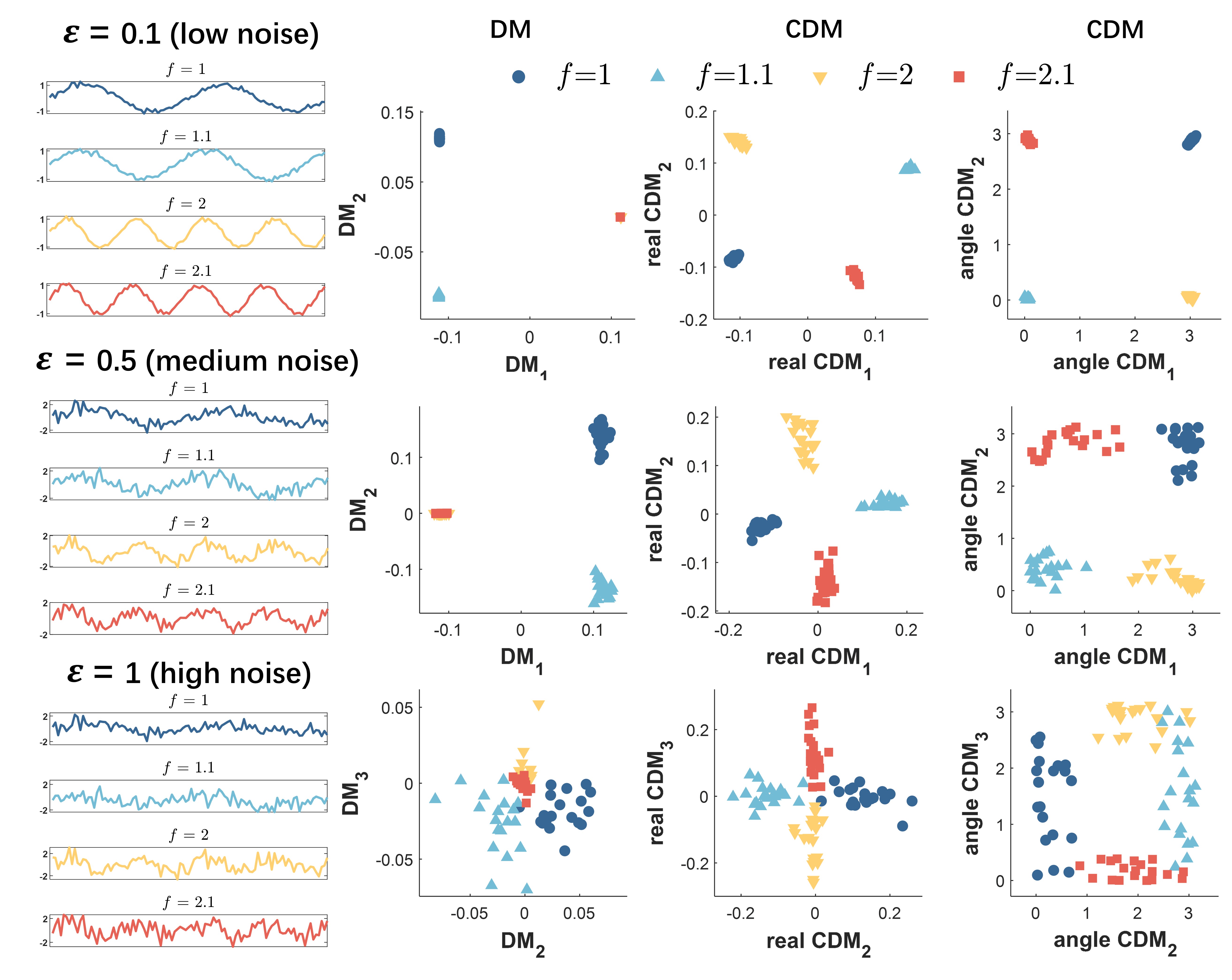}
    \caption{\footnotesize Low-dimensional representations of DM and CDM for a simulated time series under different noise levels, low noise $\varepsilon=0.1$ (SNR $\approx 17$ dB), medium noise $\varepsilon=0.5$ (SNR $\approx 3$ dB), and high noise $\varepsilon=1$ (SNR $\approx -3$ dB). Under low noise, the first two dimensions of DM capture only low-frequency components, whereas CDM reveals a clear four-cluster structure. As noise increases, the embeddings of both methods become more scattered, while CDM still shows a clear four-cluster partition even when the noise power exceeds the signal power (SNR $<0$).}
    \label{time_series}
\end{figure}
This example aims to show that the CDM method can find effective complex harmonic maps for signals with similar frequencies and these maps are robust under different signal-to-noise ratios (SNR). We first generate four sinusoidal signals with Gaussian noise of different frequencies, with 20 samples per frequency,
\begin{displaymath}
x_{i,\mathrm{c}}\left( t \right) =\sin \left( 2\pi f_{\mathrm{c}}t \right) +\eta _i\left( t \right) ,\,\,  \eta _i\left( t \right) \sim \mathcal{N} \left( 0,\varepsilon ^2 \right),
\end{displaymath}
where $i\in \left\{ 1,2,\cdots ,20 \right\} , \mathrm{c}\in \left\{ 1,2,3,4 \right\} 
$ and $f_1=1,f_2=1.1,f_3=2,f_4=2.1$. 

We consider three noise levels: $\varepsilon = 0.1$ (SNR $\approx 17$ dB), $\varepsilon = 0.5$ (SNR $\approx 3$ dB), and $\varepsilon = 1$ (SNR $\approx -3$ dB), corresponding to low, medium, and high noise regimes. Both the kernel bandwidth $\sigma$ and the complex parameter $\omega$ are selected via grid search, while the same $\sigma$ is used for both methods. The results under different noise levels are shown in Fig.~\ref{time_series}.

Fig.~\ref{time_series} shows that the classic real-valued kernel struggles to separate the four classes using only the first two diffusion components even under low-noise conditions, as it primarily distinguishes between low- and high-frequency components. In contrast, CDM leverages both phase and magnitude information in the complex embedding, which appears to better reveal the four-cluster structure. As noise increases, the embeddings of both methods become more dispersed, while CDM still tends to preserve a clearer cluster structure. When the noise power exceeds the signal power, the embeddings are significantly degraded, yet CDM continues to show relatively improved separation compared to traditional DM. These observations suggest that CDM may offer enhanced robustness over real-valued kernel methods for signals with closely spaced frequency components.

\begin{figure}[htbp]
    \centering
    \includegraphics[width=0.95\textwidth]{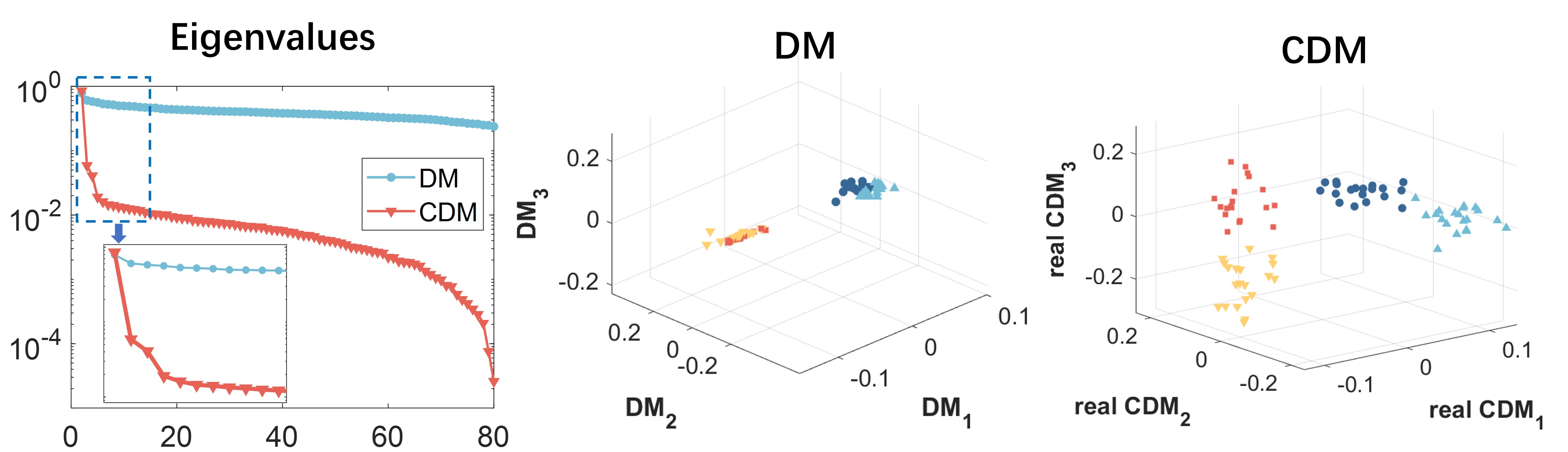}
    \caption{\footnotesize Eigenvalues and the first three dominant maps of DM and CDM for the simulated time series under high noise $\varepsilon=1$ (SNR $\approx -3$ dB). The eigenvalue gaps of CDM are more pronounced than those of DM, with the first three components contributing a larger proportion of the total variance. The corresponding three-dimensional embeddings reveal a clear four-cluster structure. In contrast, the real-valued kernel produces a first map that primarily separates low- and high-frequency components, while the subsequent eigenvalues show no evident gaps.}
    \label{eigenvalues_time_series}
\end{figure}

Fig.~\ref{eigenvalues_time_series} illustrates the eigenvalues and the first three dominant maps of DM and CDM applied to the simulated time series under high noise ($\varepsilon=1$, SNR $\approx -3$ dB). Examining the eigenvalue spectra provides insight into the quality of the embeddings. Larger gaps tend to indicate a more robust separation of dominant modes and a clearer identification of the underlying structure. The results indicate that CDM exhibits more pronounced gaps than DM, with the first three components accounting for a larger proportion of the total variance, suggesting a clearer identification of the dominant structures. The corresponding three-dimensional embeddings further support this observation, where CDM reveals well-separated clusters associated with the four frequency categories. In contrast, the real-valued DM kernel provides only partial separation between low- and high-frequency components in the first map, while the remaining spectrum shows limited separation, indicating a less informative representation. Overall, these results suggest that CDM may offer improved robustness to noise and better preserve the intrinsic structure of signals with closely spaced frequency components compared to traditional DM.

\subsection{FMRI data analysis}

Nonlocal distributed communication in the brain has been shown to play an important role in information transfer~\cite{deco2021rare,deco2025complex}. This experiment aims to verify that CDM can capture such nonlocality in spatiotemporal brain dynamics.

We study resting-state functional magnetic resonance imaging (fMRI) signals from 100 participants in the Human Connectome Project (HCP)~\cite{van2013wu}. A common brain atlas~\cite{glasser2016multi} is applied to each participant after standard preprocessing, resulting in 379 brain regions and 1200 time points. The experimental design follows the framework proposed by Deco et al.~\cite{deco2025complex}. 

Two metrics are considered to quantitatively evaluate the learned embeddings, functional connectivity (FC) reconstruction and edge-centric metastability (ECM). Specifically, we compare FC from the original time series with that reconstructed from the embeddings for each participant, and perform evaluations across kernel bandwidths, values of $\omega$, and diffusion steps to assess robustness. In addition, ECM is employed to quantify the ability to capture fine-scale temporal dynamics. It measures the consistency of edge-level dynamics between the source space and the embedding space, reflecting how effectively transient edge-wise patterns are preserved. Detailed experimental settings, as well as the computation of FC and ECM, are provided in Appendix~\ref{appendix_fMRI}. 

Fig.~\ref{fMRI}~(A) shows the correlation and reconstruction error between the reconstructed and source FC matrices across different kernel bandwidths $\sigma^2$ and diffusion steps $t$. To provide a clearer comparison, Fig.~\ref{fMRI}~(B) further reports the FC reconstruction performance at the selected optimal kernel bandwidth $\sigma$ identified in (A).

From the results, it can be observed that as the diffusion step $t$ increases, the real-valued kernel, which emphasizes local connections, tends to produce more homogeneous representations, leading to degraded performance. In contrast, the complex-valued kernel incorporates nonlocal interactions, and the presence of the imaginary component slows the convergence toward isotropy, enabling the capture of finer dynamical structures. As a result, it maintains relatively higher correlation and lower reconstruction error across diffusion steps.

We further evaluate the ability of PCA, DM, and CDM to capture nonlocal brain dynamics using ECM~\cite{deco2025complex,faskowitz2020edge}. Specifically, we compute the correlation between ECM in the source space and that in the embedding space, as shown in Fig.~\ref{fMRI}~(C).

The results indicate that linear methods fail to capture complex spatiotemporal patterns. Under the same conditions, the family of $\omega$-parameterized complex kernels generally achieves higher correlations than real-valued kernels, suggesting that CDM may better preserve the underlying spatiotemporal dynamics.

\begin{figure}[!htbp]
    \centering
    \includegraphics[width=1.0\textwidth]{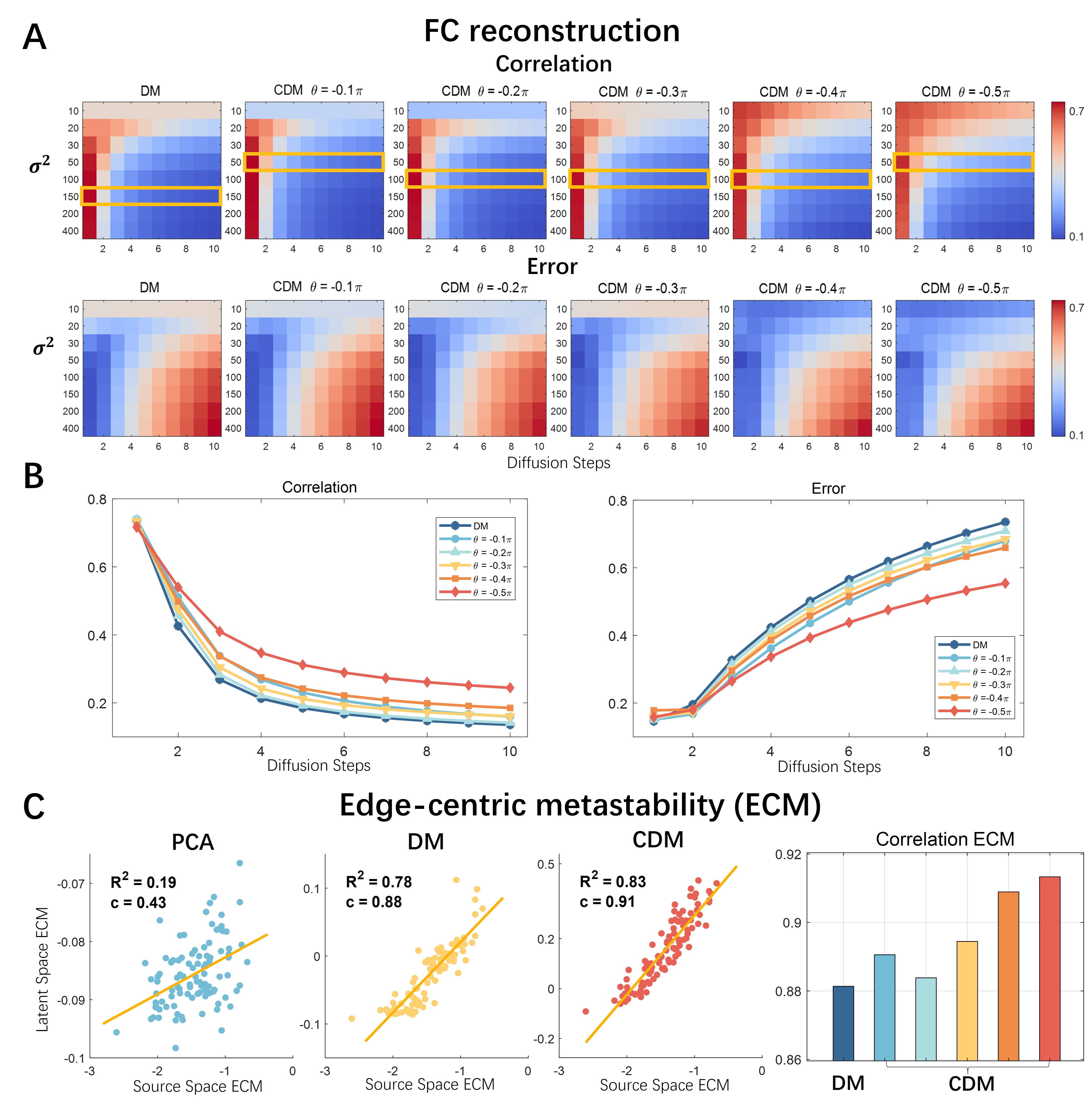}
    \caption{\footnotesize Comparison of functional connectivity (FC) reconstruction and edge-centric metastability (ECM) derived from resting-state fMRI signals of 100 participants from the Human Connectome Project~\cite{van2013wu} under PCA, DM, and CDM frameworks. (A) Correlation and reconstruction error between reconstructed and source FC matrices in the learned low-dimensional manifolds, evaluated over kernel bandwidth $\sigma^2$ and diffusion step $t$. (B) FC reconstruction performance at the selected optimal kernel bandwidth $\sigma$ (highlighted in (A)). CDM employs complex-valued kernels parameterized by $\omega = e^{\mathbf{i}\theta}$ and generally achieves higher correlation and lower error, particularly at larger diffusion steps. (C) Comparison of global spatiotemporal dynamics using ECM. Scatter plots show participant-wise correlations between source and manifold spaces, and the bar plot summarizes overall performance. CDM yields higher correlations than PCA and DM.} 
    \label{fMRI}
\end{figure}

\subsection{EEG sleep stage clustering and classification}\label{exp_eeg}
Accurate identification of sleep stages plays a crucial role in the assessment of sleep quality and the diagnosis of sleep-related disorders. In this experiment, multi-class clustering and classification are conducted on a publicly available sleep stage dataset to evaluate the effectiveness of the proposed CDM method.

We use the ISRUC-S3 dataset~\cite{khalighi2016isruc}, which contains overnight polysomnography (PSG) signals recorded from 10 healthy subjects. The recordings are annotated into five sleep stages, namely Wake, N1, N2, N3, and rapid eye movement (REM). In this experiment, 10 channels are selected from the PSG recordings, including 6 electroencephalography (EEG) channels, 2 electrooculography (EOG) channels, 1 electromyography (EMG) channel, and 1 electrocardiography (ECG) channel. Following the standard processing pipeline, all signals are downsampled from 200 Hz to 100 Hz. The final 30 seconds records of each epoch are excluded for subsequent analysis. Specific explanations for each sleep stage and more details about the ISRUC-S3 dataset are provided in Appendix ~\ref{Dataset_description}. The experimental procedure includes the following four stages.

\textbf{Order-$p$ data stacking.} For each sleep EEG signal, only ten channels are available in the spatial domain, which may lead to insufficient spectral complexity. Order-$p$ data stacking expands the effective spatial dimensionality by incorporating time-delayed copies of the original signals. This idea is rooted in delay-coordinate embedding~\cite{packard1980geometry,takens2006detecting}. $p$ denotes the stacking order.

\textbf{CDM computation.} CDM is then applied along the temporal dimension to extract nonlinear embeddings that capture interactions across channels. When no reliable prior knowledge is available, grid search remains a practical strategy for hyperparameter selection. 

\textbf{Embedding alignment.} Since the low-dimensional embeddings obtained from different samples are not naturally aligned in the same space, an embedding alignment step is required. For each sample $i$, we solve the following optimization problem:
\[
\underset{\boldsymbol{O}_i}{\min}\; \left\| \boldsymbol{E}_i \boldsymbol{O}_i - \boldsymbol{E}_{\mathrm{ref}} \right\|_F^2 ,
\]
where $E_{\mathrm{ref}}$ is the low-dimensional embedding of a selected reference sample, $E_i$ is the embedding of sample $i$, and $O_i$ is a unitary matrix. This procedure aligns all embeddings into a common space. 

\textbf{Downstream analysis.} The analysis considers two paradigms: unsupervised learning~\cite{hastie2008unsupervised} and supervised learning~\cite{cunningham2008supervised}.
\begin{itemize}
    \item \textbf{Unsupervised learning.} The K-means~\cite{macqueen1967multivariate} algorithm is applied to cluster all samples in the embedding space jointly across subjects, focusing on the global structure of the learned features. The number of clusters is set to 5, and all parameters are kept at their default values.
    
    \item \textbf{Supervised learning.} A linear-kernel SVM classifier is employed for classification, as CDM extracts nonlinear features. The \textbf{intra-subject setting} is adopted, where each subject’s data are split into training, validation, and testing sets with a ratio of $8{:}1{:}1$, and the penalty parameter is selected based on validation performance. A 10-fold cross-validation scheme is adopted across subjects, where each fold corresponds to one subject.
\end{itemize}

For the clustering baselines, several classical dimensionality reduction and manifold learning methods are selected as listed in~\ref{baselines_clustering}. All methods follow the same unsupervised pipeline as CDM, including higher-order data stacking, embedding alignment, and K-means clustering with identical parameter settings. Hyperparameters are listed in Table~\ref{tab:clustering_comparison} and determined via grid search. Specifically, $\sigma \in [0.5,5]$, $\theta \in [-0.5\pi,-0.1\pi]$ (controlling $\omega$), and the neighborhood size is chosen from $[6,8,12]$.

Figure~\ref{eeg_cluster_result} presents the clustering results on the ISRUC-S3 dataset~\cite{khalighi2016isruc}. As shown in Fig.~\ref{eeg_cluster_result}~(A), both CDM and DM achieve optimal performance at embedding dimension $d=5$, and performance slightly degrades as $d$ further increases, suggesting that excessively high-dimensional embeddings may introduce redundant or noisy components. CDM consistently outperforms DM across different embedding dimensions and stacking orders, demonstrating its robustness to embedding choices.

\begin{figure}[!htbp]
    \centering
    \includegraphics[width=1.0\textwidth]{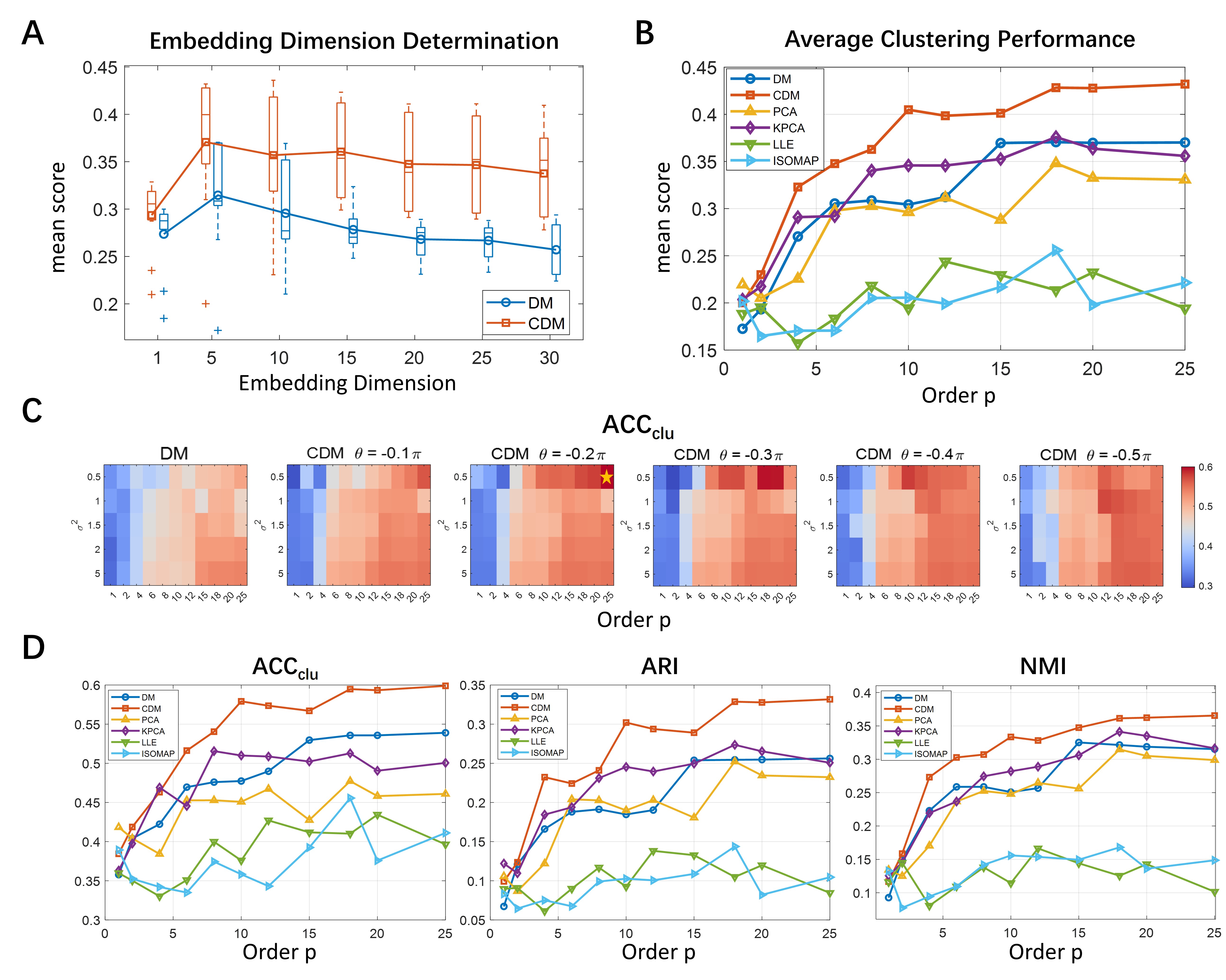}
    \caption{\footnotesize
Clustering results on the ISRUC-S3 dataset~\cite{khalighi2016isruc}. The dataset contains overnight polysomnography (PSG) recordings from 10 subjects, annotated into five sleep stages: Wake, N1, N2, N3, and rapid eye movement (REM). The evaluation metrics $\mathrm{ACC_{clu}}$, ARI, and NMI are described in~\ref{Evaluation_matrices_clu}.
(A) Mean clustering scores of CDM and DM under varying embedding dimensions and stacking orders $p$. CDM consistently achieves higher scores than DM, and both methods attain optimal performance at an embedding dimension of 5.
(B) Comparison of mean clustering scores between CDM and baseline methods at the optimal embedding dimension, showing the superiority of CDM.
(C) Clustering accuracy $\mathrm{ACC_{clu}}$ of CDM and DM under different $p$ and kernel bandwidth $\sigma$ at the optimal embedding dimension. CDM achieves higher accuracy across a wide range of parameters, with the best result marked by a yellow pentagram.
(D) Comparison of $\mathrm{ACC_{clu}}$, ARI, and NMI between CDM and baseline methods at the optimal embedding dimension, where CDM outperforms all baselines across the three metrics.
}
\label{eeg_cluster_result}
\end{figure}
Based on this observation, the embedding dimension is fixed at $d=5$ in the subsequent experiments. Under this setting, CDM achieves the best average performance among all compared methods, confirming its advantage over real-valued baselines in capturing the intrinsic data structure, as illustrated in Fig.~\ref{eeg_cluster_result}~(B).

The parameter sensitivity analysis in Fig.~\ref{eeg_cluster_result}~(C) further shows that CDM maintains stable performance across a wide range of kernel bandwidths $\sigma$ and stacking orders $p$. In particular, the clustering accuracy is relatively insensitive to the complex-valued kernel parameter $\omega$, indicating that CDM does not rely heavily on fine-tuning of phase-related parameters.

Finally, Fig.~\ref{eeg_cluster_result}~(D) demonstrates that CDM consistently achieves the best performance across all evaluation metrics. Kernel-based methods (DM and KPCA) generally outperform linear methods such as PCA, highlighting the importance of nonlinear feature extraction. In contrast, graph-based methods (LLE and ISOMAP) exhibit inferior and less stable performance, likely due to their sensitivity to neighborhood graph construction.

Table~\ref{tab:clustering_comparison} shows the clustering performance comparison between CDM and benchmark methods under the fixed embedding dimension and optimal stacking order. The FDR metric refers to the Fisher discriminant ratio (FDR) calculated after visualizing the embeddings of each method in three-dimensional space using t-SNE ~\cite{van2008visualizing}. For details on the calculation, refer to ~\ref{Evaluation_matrices_clu}. A higher FDR value indicates a relatively higher degree of inter-class separation and a relatively better degree of intra-class aggregation.

The results in Table~\ref{tab:clustering_comparison} indicate that CDM achieves the best results compared to other benchmark methods in $\mathrm{ACC_{clu}}$, ARI, NMI, the mean of the clustering metrics, and the degree of intra-class aggregation in the embedding visualization. This indicates that the introduction of complex-valued similarity metrics in CDM could indeed extract more discriminative information beneficial for distinguishing sleep stages compared to real-valued methods.

\begin{table}[htbp]
\centering
\small
\setlength{\tabcolsep}{4pt}
\renewcommand{\arraystretch}{1}
\caption{\footnotesize{Comparison of clustering performance between CDM and baselines at embedding dimension of 5 under optimal stacking orders. Evaluation metrics are defined in Appendix~\ref{Evaluation_matrices_clu}. The mean score is computed as $(\mathrm{ACC_{clu}}+\mathrm{ARI}+\mathrm{NMI})/3$. Bold and underline indicate the best and second-best results, respectively.}}
\label{tab:clustering_comparison}
\begin{tabular}{ccccccc}
\toprule
Method & Hyperparameter setting & $\mathrm{ACC_{clu}}$ & ARI & NMI & Mean score & FDR\\
\midrule
PCA~\cite{wold1987principal}   & \textbf{--}                    & 0.477 & 0.252 & 0.315 & 0.348 & 0.092 \\
KPCA~\cite{scholkopf1997kernel}   &  bandwidth $\sigma$         & 0.516 &  \underline{0.274} &  \underline{0.341} &  \underline{0.376} & 0.068\\
LLE~\cite{roweis2000nonlinear}    &  Number of neighbors $k$              & 0.434 & 0.138 & 0.166 & 0.244 & 0.119\\
ISOMAP~\cite{tennenbaum2000global} &  Number of neighbors $k$              & 0.456 & 0.144 & 0.168 & 0.256 & 0.151\\
DM~\cite{coifman2005geometric}     &  bandwidth $\sigma$         &  \underline{0.539} & 0.256 & 0.325 & 0.371 & \underline{0.189}\\
\textbf{CDM}    &  bandwidth $\sigma$, complex kernel parameter $\omega$  & \textbf{0.599} & \textbf{0.332} & \textbf{0.365} & \textbf{0.432} & \textbf{0.309}\\
\bottomrule
\end{tabular}
\end{table}

Fig.~\ref{eeg_class_result} presents the classification performance on the ISRUC-S3 dataset~\cite{khalighi2016isruc} under the intra-subject setting. As shown in Fig.~\ref{eeg_class_result}~(A), the grid search results exhibit no clear pattern, suggesting that the optimal configuration is data-dependent in the absence of prior knowledge. As shown in Fig.~\ref{eeg_class_result}~(B), moderate stacking orders $p$ improve generalization, while larger orders lead to overfitting, which can be interpreted through the bias--variance trade-off, where overly representations tend to capture dataset-specific artifacts rather than discriminative sleep-stage patterns. Finally, Fig.~\ref{eeg_class_result}~(C) shows that N1 is classified with lower accuracy and is often confused with N2. This is consistent with the fact that N1 and N2 correspond to transitional stages with similar physiological characteristics, making them inherently difficult to distinguish.

\begin{figure}[ht]
    \centering
    \includegraphics[width=0.95\textwidth]{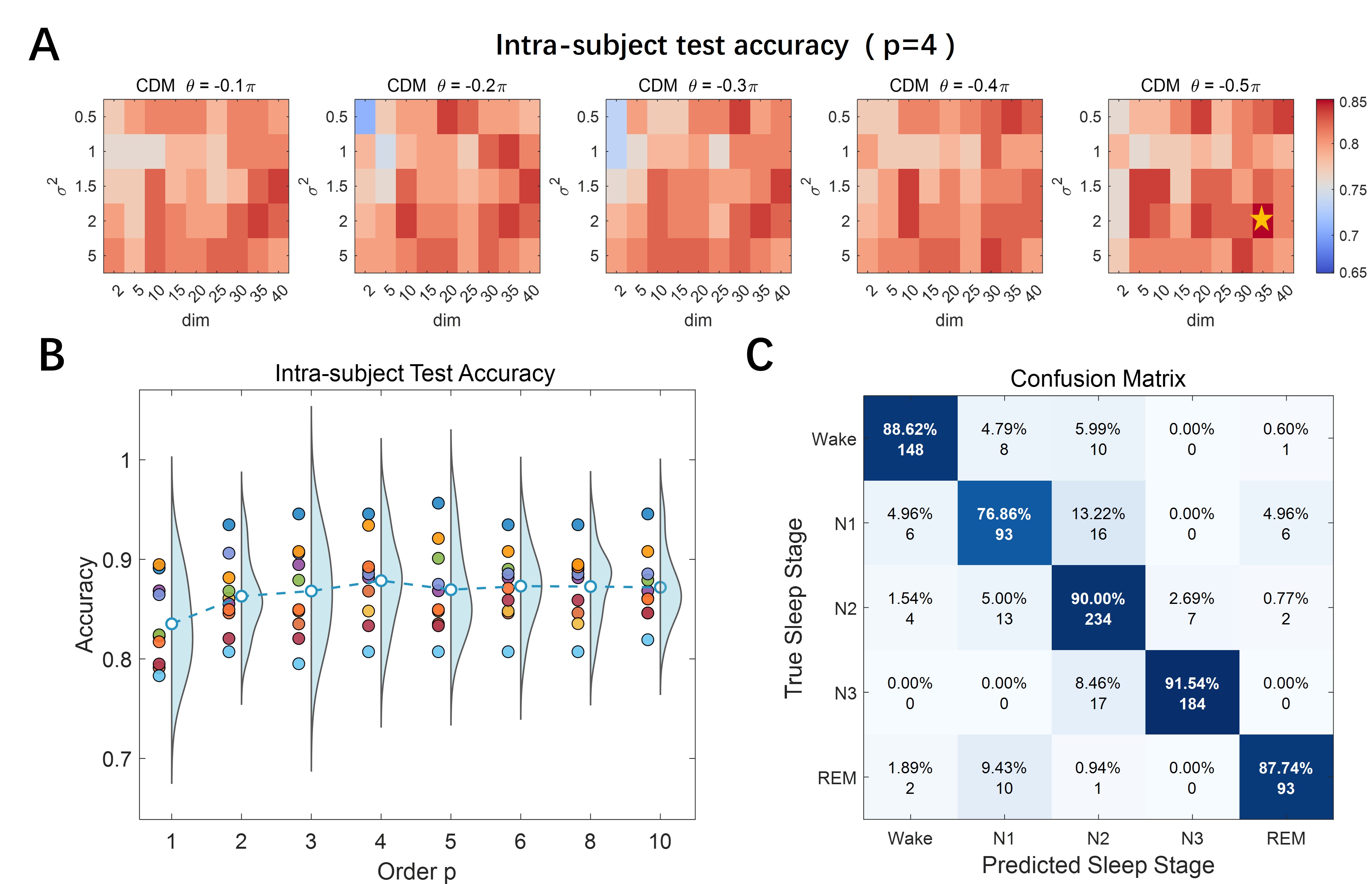}
    \caption{\footnotesize
10-fold cross-validation results of CDM under the intra-subject setting on the ISRUC-S3 dataset~\cite{khalighi2016isruc}.
(A) Test accuracy $\mathrm{ACC_{cls}}$ for one fold at stacking order $p=4$, where \textit{dim} denotes the embedding dimension. The best result is marked with a star.
(B) Best test accuracy of each fold under different stacking orders $p$, with the mean accuracy indicated by hollow markers, showing the optimal order at $p=4$.
(C) Confusion matrix corresponding to the overall results at $p=4$.
}
\label{eeg_class_result}
\end{figure}

Table~\ref{tab:isruc_s3_comparison} reports the classification performance under the intra-subject setting. CDM consistently outperforms four deep neural network baselines across all metrics. This advantage could be attributed to two main factors: (1) limited training data per subject, under which complicated neural networks tend to overfit due to insufficient supervision, while CDM provides a more compact and data-efficient representation that generalizes better in small-sample regimes; and (2) the low-dimensional neural manifold underlying neural activity~\cite{gallego2017neural,perich2025neural}, which is more naturally captured by the manifold learning framework. By explicitly exploiting this intrinsic geometric structure, CDM can better preserve the underlying temporal and dynamical patterns, enabling more effective modeling of subject-specific neural dynamics.

In addition to classification performance, the computational efficiency of different methods is further compared. As illustrated in Fig.~\ref{runtime_intra}, CDM achieves a significantly lower runtime than all deep neural network baselines. Specifically, while neural network models require tens to over one hundred seconds due to their complex architectures and training procedures, CDM completes the entire pipeline within only a few seconds.

This efficiency gain stems from the lightweight design of CDM, which avoids iterative backpropagation and large-scale parameter optimization. As shown in the runtime breakdown, the majority of computation is concentrated in the order-$p$ data stacking and CDM embedding stages. Thus, CDM achieves both superior classification performance and higher computational efficiency.

\begin{table*}[htbp]
\centering
\small
\setlength{\tabcolsep}{4pt}
\caption{\footnotesize{Comparison of classification performance between CDM and baselines under the intra-subject setting. Evaluation metrics are defined in Appendix~\ref{Evaluation_matrices_cls}. The bolded and underlined results represent the best and second-best, respectively, for partial results.}}
\label{tab:isruc_s3_comparison}
\setlength{\tabcolsep}{6pt}
\begin{tabular}{c c c c c c c c c c}
\toprule
\multirow{2}{*}{\makecell{Setting}}&
\multirow{2}{*}{Method} &
\multicolumn{3}{c}{Overall results} &
\multicolumn{5}{c}{F1-score for each class}\\
\cmidrule(lr){3-5} \cmidrule(lr){6-10}
& & $\mathrm{ACC_{cls}}$ & F1 & Kappa & Wake & N1 & N2 & N3 & REM  \\
\midrule

\multirow{5}{*}{\makecell{Intra-\\subject}}
& SalientSleepNet~\cite{ijcai2021p360}    
& 0.750 & 0.714 & 0.674 & 0.814 & 0.510 & 0.728 & 0.904 & 0.613\\
& MSTGCN~\cite{jia2021multi}              
& 0.748 & 0.724 & 0.674 & 0.812 & 0.435 & 0.794 & 0.821 & 0.795\\
& AttenNet~\cite{eldele2021attention}    
& 0.851 & 0.821 & 0.805 & \underline{0.900} & 0.605 & \underline{0.860} & \underline{0.934} & 0.804\\
& SleepPrintNet~\cite{jia2021sleepprintnet} 
& \underline{0.853} & \underline{0.839} & \underline{0.808} & 0.882 & \underline{0.684} & 0.851 & 0.927 & \underline{0.850} \\
& \textbf{CDM}                    
& \textbf{0.880} & \textbf{0.874} & \textbf{0.845} & \textbf{0.905} & \textbf{0.759} & \textbf{0.870} & \textbf{0.939} & \textbf{0.894}\\



\bottomrule
\end{tabular}
\end{table*}
Finally, it is worth emphasizing that CDM is a general-purpose framework rather than a model explicitly designed for EEG sleep stage classification. The competitive performance reported here is achieved without task-specific architectural tuning, suggesting strong potential for applicability across a wide range of datasets and problem domains beyond sleep analysis.

\begin{figure}[htbp]
    \centering
    \includegraphics[width=0.85\textwidth]{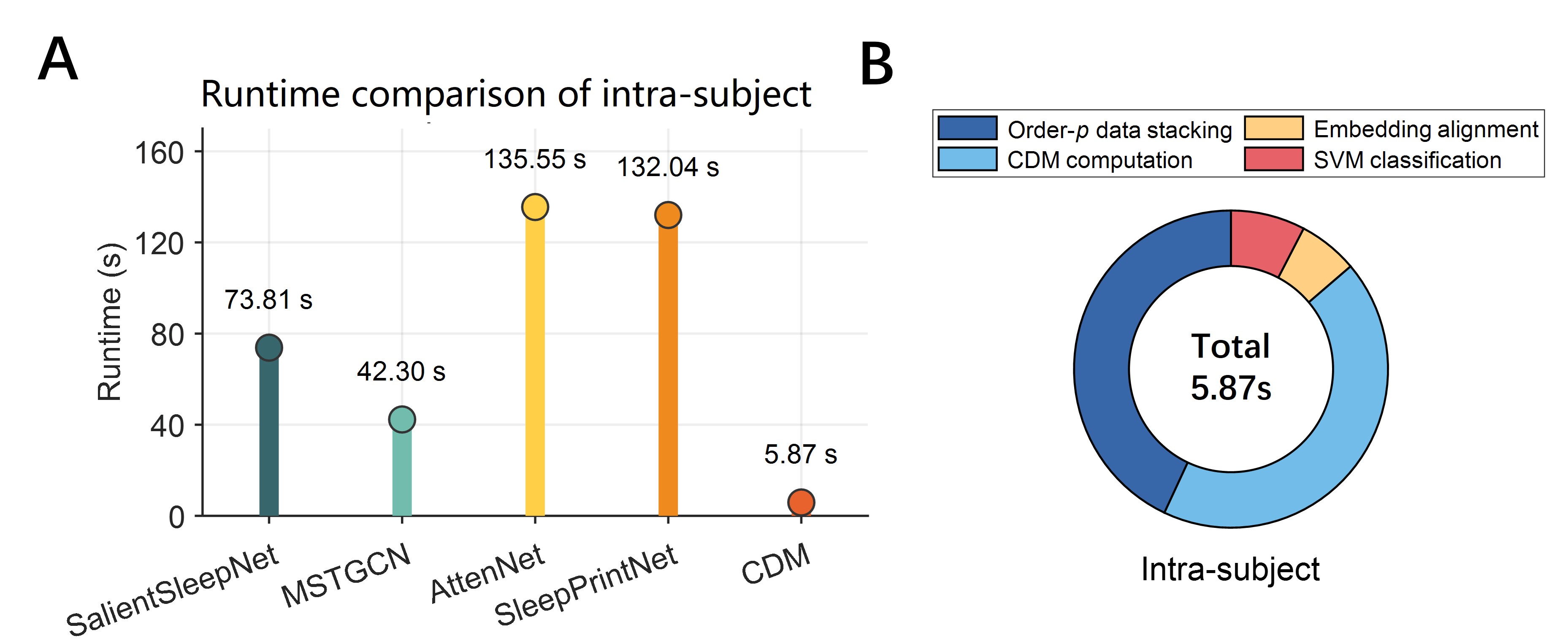}
    \caption{\footnotesize{Display of runtime on intra-subject setting. Runtime refers to the average runtime of 10-fold experiments with their respective optimal parameters. (A) Comparison of the runtime of CDM with baseline neural network-based methods. (B) Breakdown of CDM’s runtime across its four phases, Order-$p$ data stacking, CDM computation, embedding alignment, and SVM classification.} } 
    \label{runtime_intra}
\end{figure}
To further evaluate the generalization ability of CDM cross subjects, additional classification experiments are conducted under the cross-subject setting, focusing on its ability to transfer learning. The results show that CDM outperforms several machine learning and neural network baselines, while still maintaining clear computational efficiency. Due to space limitations, detailed results and analysis are presented in Appendix~\ref{appendix_eeg_cross_result}.

\section{Conclusion} \label{conclusion}
Inspired by the Gaussian kernel induced by the heat equation and the complex-valued Schrödinger kernel induced by the Schrödinger equation, we propose a family of unified generalized $\omega$-parameterized complex-valued kernels. We further apply it to manifold learning and propose Complex Diffusion Maps (CDM), a novel algorithm that aims to reveal low-dimensional complex harmonic representations of high-dimensional data. We also explain the optimal complex embedding by a discrete form of optimization problem, where points with similar global connection angles in complex space are expected to stay as close as possible in the embedded space. 

Extensive experiments demonstrate that CDM is both effective and robust. The proposed complex-valued formulation improves sample separability and preserves intrinsic structures even under noise. Experiments on resting-state fMRI data show that CDM can accurately capture functional connectivity patterns and global nonlocal dynamics. On EEG-based sleep stage data, CDM achieves competitive clustering and classification performance, while offering clear advantages in computational efficiency over classical manifold learning and deep neural network methods.

The CDM framework still has some possible extensions and generalizations. One potential direction is the development of an adaptive method to optimize the complex-valued kernel parameter $\omega$ for improved performance. Additionally, incorporating complex-valued neural networks may enhance learning within this framework. From an optimization perspective, CDM is closely related to the Hermitian Laplacian, which plays a crucial role in graph neural networks~\cite{kipf2017semisupervised}. Applying the Hermitian Laplacian to extract richer possible information from complex low-dimensional manifolds presents a promising avenue for further research.

\section*{Funding} 
This work was partially supported by the Science and Technology Commission of Shanghai Municipality (No. 23ZR1403000) and the National Natural Science Foundation of China (No. 12471481). The research of M. Ng was supported by the GDSTC: Guangdong and Hong Kong Universities “1+1+1” Joint Research Collaboration Scheme UICR0800008-24, National Key Research and Development Program of China under Grant 2024YFE0202900, RGC GRF 12300125.


\section*{Conflict of Interest}
The authors declare that they have no conflict of interest.

\appendix
\section*{Appendix}
\addcontentsline{toc}{section}{Appendix}
\section{Out-of-sample extensions}\label{Out-of-sample}
Algorithm~\ref{algorithm 1} outlines the computational steps of CDM. Let $d$ be the feature dimension and $N$ the number of samples, with $d \ll N$. The time and space complexities are $\mathcal{O}(N^2 d + N^3)$ and $\mathcal{O}(N^2)$, respectively. It is prohibitive for large datasets. In real-valued kernels, sparsity is often introduced by thresholding small entries, but this is unsuitable for complex kernels since magnitude-based truncation may lose phase information. To overcome this, we adopt a Nyström-based extension~\cite{williams2000using} for efficient out-of-sample computation.

Recalling the derivation of Theorem~\ref{theorem_CDM}, the diffusion operator $\mathbf{A}$ is compact and self-adjoint, satisfying 
$
\mathbf{A}\bm{\phi }_n = \lambda _n \bm{\phi }_n,
$
where $\{\lambda_n\}_{n\in \mathbb{N}}$ are the eigenvalues and $\{\bm{\phi}_n\}_{n\in \mathbb{N}}$ are the corresponding eigenfunctions. 
Combining with Eq.~\eqref{eq_operator_A}, for any eigenvalues $\lambda_n > 0$, we have
\begin{displaymath}
\lambda _{n}^{1/2}\bm{\phi }_n=\frac{\mathbf{A}\bm{\phi }_n}{\lambda _{n}^{1/2}}=\int_{\Gamma}{\frac{a\!\left( \bm{x},\bm{y} \right)}{\lambda _{n}^{1/2}}\bm{\phi }_n\!\left( \bm{y} \right) \,\mathrm{d}\mu \!\left( \bm{y} \right)}.
\end{displaymath}
For a newly given data point $\bm{z}$, the Nyström method~\cite{williams2000using} extends this expression to
\begin{displaymath}
\lambda _{n}^{1/2}\bm{\phi }_n\!\left( \bm{z} \right) = \int_{\Gamma}{\frac{a\!\left( \bm{z},\bm{y} \right)}{\lambda _{n}^{1/2}}\bm{\phi }_n\!\left( \bm{y} \right) \,\mathrm{d}\mu \!\left( \bm{y} \right)}.
\end{displaymath}
Approximating the integral by a discrete sum yields
\begin{displaymath}
\lambda _{n}^{1/2}\bm{\phi }_n\!\left( \bm{z} \right) \approx \sum_{i=1}^N{\frac{a\!\left( \bm{z},\bm{y}_i \right)}{N\lambda _{n}^{1/2}}\,\bm{\phi }_n\!\left( \bm{y}_i \right)}.
\end{displaymath}
Therefore, the first $k$ complex diffusion maps with $t$ diffusion steps for an out-of-sample point $\bm{z}$ can be expressed as
\begin{equation}\label{eq_out_of_sample}
\bm{\psi }_{k}(\bm{z}) = \sum_{i=1}^N 
\left[
\begin{matrix}
\dfrac{a\!\left( \bm{z},\bm{y}_i \right)}{\lambda _{1}^{1/2}}\,\bm{\phi }_1\!\left( \bm{y}_i \right), &
\dfrac{a\!\left( \bm{z},\bm{y}_i \right)}{\lambda _{2}^{1/2}}\,\bm{\phi }_2\!\left( \bm{y}_i \right), &
\cdots, &
\dfrac{a\!\left( \bm{z},\bm{y}_i \right)}{\lambda _{s}^{1/2}}\,\bm{\phi }_s\!\left( \bm{y}_i \right)
\end{matrix}
\right].
\end{equation}
Note that the factor $1/N$ has the same effect on each map and each new data, so it does not affect the embedded representation and is omitted here. 

Let $\bm{\varPhi} = \left[ \bm{\phi}_1, \bm{\phi}_2, \cdots, \bm{\phi}_s \right]$ 
and $\bm{\varLambda} = \operatorname{diag}(\lambda_1, \lambda_2, \cdots, \lambda_s)$ 
denote the first $k$ eigenvectors and eigenvalues of the kernel matrix $\bm{A}$ in Algorithm~\ref{algorithm 1} computed by the existing $N$ samples, respectively. $\bm{A}_{\mathrm{G}}$ denotes the kernel matrix between the generalized points and the existing points with size of $N_{\mathrm{G}} \times N$. Then, Eq.~\eqref{eq_out_of_sample} can be expressed in matrix form, where the first $k$ eigenmaps for the $N_{\mathrm{G}}$ new samples are denoted by $\bm{\varPhi}_{\mathrm{G}}^{t}$:
\begin{equation}\label{eq_matrix_out_of_sample}
    \bm{\varPhi}_{\mathrm{G}} = \bm{A}_{\mathrm{G}} \bm{\varPhi} \bm{\varLambda}^{-\frac{1}{2}}.
\end{equation}

In subsequent experiments, we combine the expression of out-of-sample extensions and the Geometric Harmonics idea to reconstruct the generalized dataset, in order to illustrate the ability of CDM to extract latent information. See Appendix~\ref{appendix_fMRI} for details.
\section{Artificial point set construction}\label{Artificial_pointset}
In this appendix, we provide a detailed description of the construction process for the artificial dataset in Section~\ref{exp_artificial}, which extends the illustrative example in Section~\ref{three_point_example} to a larger-scale setting. The goal is to simulate a scenario in which different classes may overlap in the real domain but remain separable in a complex feature space through a joint encoding of amplitude and phase information. By explicitly modeling both magnitude and phase relationships, this setup allows us to investigate the performance of complex-valued methods in capturing cluster structures that may not be evident in the original real-valued data.

Specifically, instead of explicitly generating feature vectors, we directly construct a complex-valued relational matrix that implicitly encodes a three-cluster structure in the complex domain.The clusters are characterized by their associated means $\mu_1 = 1 + \mathbf{i}, \mu_2 = 1 - \mathbf{i}, \mu_3 = -1 + 2\mathbf{i}$, with each cluster containing $n = 100$ samples, for a total of $N = 300$ samples. The off-diagonal entries of the complex matrix $\boldsymbol{D} \in \mathbb{C}^{N\times N}$ are generated as
\begin{displaymath}
\boldsymbol{D}_{ij} = \mu_c + \eta (z_{ij} + \mathbf{i}\,y_{ij}),
\qquad
z_{ij}, y_{ij} \sim \mathcal{N}(\mathbf{0}, \mathbf{I}), \quad i \neq j,
\end{displaymath}
where $c$ denotes the cluster of sample $j$, and $\eta = 2$ controls the variability of the entries. The resulting matrix $\boldsymbol{D}$ encodes both amplitude and phase information and can be interpreted as a complex-valued affinity matrix describing pairwise couplings between samples. To obtain a real-valued distance representation that integrates both characteristics, we combine the normalized phase and magnitude components with weights $\alpha$ and  $\beta$, corresponding to their relative importance,
\begin{displaymath}
\tilde{\boldsymbol{D}}_{ij}
= \alpha \cdot 
\frac{|\boldsymbol{D}_{ij}|}{|\boldsymbol{D}|_{\max}}
+ \beta \cdot 
\frac{\mathrm{mod}\!\left(\arg(\boldsymbol{D}_{ij}), \, 2\pi\right)}{2\pi}, \qquad \alpha=0.1, \beta =0.5.
\end{displaymath}

Here, $|\bm{D}_{ij}|$ denotes the magnitude of $\bm{D}_{ij}$, $\arg(\bm{D}_{ij})$ represents its phase angle, and $\mathrm{mod}(\cdot,2\pi)$ maps the phase into the range $[0,2\pi)$ to ensure positivity before normalization. $\max_{p,q}|\bm{D}_{pq}|$ denotes the maximum magnitude over all elements. Consequently, Consequently, the first term captures the normalized magnitude information, while the second term reflects the relative phase relationship. The weighting coefficients $\alpha$ and $\beta$ balance their contributions, reflecting the relative importance of magnitude and phase in constructing the complex similarity structure. The final symmetric matrix used for the subsequent analysis is then given by
\begin{displaymath}
\bm{D}_{\mathrm{final}} = \frac{1}{2}(\tilde{\bm{D}} + \tilde{\bm{D}}^{\top}) \in \mathbb{R}^{N\times N}.
\end{displaymath}

Based on multidimensional scaling~\cite{davison2000multidimensional}, double centering is applied to $\bm{D}_{\mathrm{final}}$ to obtain an explicit feature embedding in a Euclidean space. This results in a reconstructed dataset with 279-dimensional features. The resulting matrices are illustrated in Fig.~\ref{large_set_example}(A), where the left panel shows the reconstructed dataset. This reconstructed dataset can be viewed as points lying in three clusters within a complex-feature space, where both amplitude and phase jointly determine inter-cluster relationships. The middle panel depicts the corresponding Euclidean distance matrix $\bm{D}_{\mathrm{final}}$, and the right panel visualizes the phase-angle distribution of the complex kernel entries. 

The kernel parameter $\omega$ in CDM is chosen to be consistent with the amplitude–phase weighting ratio, so that its direction in the complex plane reflects the relative weighting between amplitude and phase components:
\begin{displaymath}
    \omega = \frac{ \alpha- \beta\mathbf{i}}{|\alpha - \beta\mathbf{i}|},
\end{displaymath}
where the negative sign takes into account the range of $\omega$ as illustrated in Fig.~\ref{complex_kernel}. 

We further compare CDM with both linear methods, including PCA and MDS, and nonlinear methods, including Kernel PCA, t-SNE, and diffusion maps (DM). All kernel-based methods use the same bandwidth setting for a fair comparison. The results are shown in Fig.~\ref{large_set_example} (B).

\section{fMRI data experiment}\label{appendix_fMRI}
\subsection{Dataset description}
Functional magnetic resonance imaging (fMRI) data measures the blood oxygen level-dependent (BOLD) signal , which provides an indirect indicator of neural activity.  We used the resting-state fMRI data from 100 participants in the Human Connectome Project (HCP) database~\cite{van2013wu}. The preprocessing followed established pipelines that included the removal of spatial artifacts, reconstruction of the cortical surface, cross-modal registration, and alignment to a standard space~\cite{glasser2013minimal}. Then a common brain region template~\cite{glasser2016multi} was applied to each participant, resulting in 379 brain regions and 1200 time stamps of fMRI signals.

This experimental design mainly follows the framework proposed by Deco et al~\cite{deco2025complex}. To assess the fidelity of manifold-based dimensionality reduction into the latent space, we compared functional connectivity (FC) computed from the original time series with FC from the reconstructed time series for each participant. We also performed systematic comparisons across different kernel bandwidths, values of the parameter $\omega$, and numbers of diffusion steps in order to evaluate robustness.

In addition, we employed the edge-centric metastability (ECM) metric to quantify how well the CDM captures fine-scale temporal dynamics. ECM measures the consistency of edge-level dynamics between the original source space and the low-dimensional embedding, thereby indicating how effectively the manifold embedding preserves transient, edge-wise patterns. 

Below we provide a detailed description of the FC and ECM evaluation metrics.
\subsection{Functional connectivity (FC)}
Following the Geometric Harmonics (GHs) framework~\cite{coifman2006geometric}, the low-dimensional embeddings of the generalized samples $\bm{\varPhi}_{\mathrm{G}}^{t}$ can be lifted back to the original data space. Each feature of the original data is treated first as a function defined in the existing samples. These functions are then extended to the new points via the eigenbasis of the kernel matrix. Due to space limitations, the detailed derivation process can be found in~\cite{papaioannou2022time,patsatzis2023data}. This leads to the following reconstruction formula:
\begin{displaymath}
    \bm{X}_{\mathrm{G}} = \mathrm{Real}(\bm{A}_{\mathrm{G}} \bm{\varPhi} \bm{\varLambda}^{-\frac{t}{2}} \bm{\varPhi}^{*} \bm{X}),
\end{displaymath}
where $\bm{X} \in \mathbb{R}^{N \times M}$ denotes the original data matrix and $\bm{X}_\mathrm{G} \in \mathbb{R}^{N_\mathrm{G} \times M}$ denotes the generalized data matrix. 
Here, the matrix $\bm{\varPhi}^{*} \bm{X}$ represents the projection of the original features onto the eigenbasis, and the left multiplication by $\bm{A}_{\mathrm{G}} \bm{\Phi} \bm{\varLambda}^{-\frac{t}{2}}$ corresponds to the Nyström-based extension as Eq.~\eqref{eq_matrix_out_of_sample} depicts, of these coefficients to the generalized points.

Following the experimental setup of Deco et al.~\cite{deco2025complex}, we explore the low-dimensional embeddings across the time dimension. Let $\bm{X}\in \mathbb{R}^{T\times M}$ denote the observed BOLD signal matrix, where $T$ is the number of time recordings and $M$ the number of brain regions. 
For each brain region $m$, we define its corresponding time series $\bm{x}_m$ as the $m$-th column of $\bm{X}$. 
The empirical functional connectivity, denoted as $\mathrm{FC}$, is computed using the Pearson correlation:
\begin{displaymath}
(\mathrm{FC})_{m,l} = \frac{\mathbb{E}_t \big[ (\bm{x}_m - \mu_m)(\bm{x}_l - \mu_l) \big]}{\sigma_m \sigma_l},
\end{displaymath}
where $\mu_m$ and $\sigma_m$ represent the mean and standard deviation of $\bm{x}_m$ across time, 
and $\mathbb{E}_t[\cdot]$ denotes the empirical average over time dimension. Similarly, the functional connectivity for the reconstructed time series $\bm{X}_\mathrm{G}$, 
denoted $\mathrm{FC}_{\mathrm{G}}$, is given by
\begin{displaymath}
(\mathrm{FC}_{\mathrm{G}})_{m,l} = \frac{\mathbb{E}_t \big[ (\hat{\bm{x}}_{m} - \hat{\mu}_m)(\hat{\bm{x}}_l - \hat{\mu}_l) \big]}{\hat{\sigma}_m \hat{\sigma}_l},
\end{displaymath}
where $\hat{\bm{x}}_m$ is the reconstructed time series corresponding to region $m$, and $\hat{\mu}_m$ and $\hat{\sigma}_m$ represent the mean and standard deviation. 
To quantify the similarity between $\mathrm{FC}$ and $\mathrm{FC}_{\mathrm{G}}$,  we consider the mean squared error of all brain regions as \emph{FC reconstruction error},
\begin{equation}\label{eq_err_fc}
    \mathrm{Err_{FC}}=\frac{1}{M^2}\sum_{m,l}^{M}[(\mathrm{FC})_{m,l}-(\mathrm{FC_G})_{m,l}]^2=\frac{1}{M^2}\|\mathrm{FC-FC_{\mathrm{G}}}\|_{\mathrm{F}}^2.
\end{equation}
The correlation between the elements of the original and generalized FC matrices is given by
\begin{displaymath}
\mathrm{Corr_{FC}}=\frac{\mathbb{E}_{m,l} \big[ \big((\mathrm{FC})_{m,l} - \mu\big)\big((\mathrm{FC_G})_{m,l} - \mu_{\mathrm{G}}\big) \big]}{\sigma \, \sigma_{\mathrm{G}}},
\end{displaymath}
where $\mathbb{E}_{m,l}[\cdot]$ denotes the empirical average over all elements of the original and generalized functional connectivity matrices,
$\mu = \mathbb{E}_{m,l}[(\mathrm{FC})_{m,l}]$ and
$\mu_{\mathrm{G}} = \mathbb{E}_{m,l}[(\mathrm{FC_G})_{m,l}]$ are the corresponding mean values,
and $\sigma$, $\sigma_{\mathrm{G}}$ denote their standard deviations.
\subsection{Edge-centric metastability (ECM)}
Edge-centric metastability has been proposed as a metric for assessing the variability in edge cofluctuation patterns, capturing fine-scale dynamics in fMRI recordings~\cite{faskowitz2020edge}. 
For the BOLD signal matrix mentioned above $\bm{X}\in \mathbb{R}^{T\times M}$, we first z-score each column across time, 
obtaining the normalized signal $\tilde{\bm{X}}$.  
From $\tilde{\bm{X}}$, we construct the edge-centered matrix 
\begin{displaymath}
    \bm{E} = [\bm{e}_1^\top, \ldots, \bm{e}_{T}^\top]^\top \in \mathbb{R}^{T \times M (M- 1)},
\end{displaymath}
where each row of the edge-centered matrix $\bm{e}_t^\top$ encodes all pairwise products of the node values at time $t$, i.e., 
the instantaneous cofluctuation pattern across edges.
For this edge representation, the functional connectivity dynamics (FCD) matrix is defined as
\begin{displaymath}
(\mathrm{FCD})_{t,s}= 
\frac{\bm{e}_{t}^\top \, \bm{e}_{s}}{\| \bm{e}_{t} \| \, \| \bm{e}_{s} \|}, 
\end{displaymath}
where $\bm{e}_{t}, \bm{e}_{s}$ denote the $t$-th and $s$-th row of $\bm{E}$, 
and $\| \cdot \|$ denotes the $\ell_2$ norm of a vector. The edge-centric metastability of the spatiotemporal signal $\bm{X}$ is then defined via the Gaussian entropy:
\begin{equation}
\label{eq:entropy}
H_{\bm{X}}=\tfrac{1}{2} \log \left( 2 \pi \, \sigma_\mathrm{FCD}^2 \right) + \tfrac{1}{2},
\end{equation}
where $\sigma_\mathrm{FCD}$ is the variance of the upper-triangular elements of FCD.
Thus, $H_{\bm{X}}$ quantifies the degree of variability in edge cofluctuation patterns. 
For each participant, we compute the edge-centric metastability in the source space, $H_{\bm{X}}$, 
using Eq.~\eqref{eq:entropy}, and in the manifold space, the reduced space obtained from the embedding representation, $H_{\bm{Y}}$.

The conservation of edge-centric metastability between source and manifold space is then quantified 
by computing the correlation across participants:
\begin{displaymath}
\mathrm{Corr_{ECM}} = 
\frac{\mathbb{E}_i \Big[ (H_{\bm{X}}^{(i)} - \mu_{\bm{X}})(H_{\bm{Y}}^{(i)} - \mu_{\bm{Y}}) \Big]}
{\sigma_{\bm{X}}\, \sigma_{\bm{Y}}},
\end{displaymath}
where $i$ indexes participants, $\mu_{\bm{X}}$ and $\mu_{\bm{Y}}$ denote the mean values of $H_{\bm{X}}^{(i)}$ and $H_{\bm{Y}}^{(i)}$, respectively, and $\sigma_{\bm{X}}$ and $\sigma_{\bm{Y}}$ denote their corresponding standard deviations.
\section{EEG sleep stage experiment}
\subsection{Dataset description}\label{Dataset_description}
The ISRUC-S3 dataset~\cite{khalighi2016isruc}, which contains overnight polysomnography (PSG) signals recorded from 10 healthy subjects, including 9 males and 1 female, aged between 30 and 58 years. The recordings are annotated into five sleep stages, namely Wake, N1, N2, N3, and rapid eye movement (REM) according to the American Academy of Sleep Medicine (AASM) standard as follows. 
\begin{itemize}
  \item \textbf{Wake}: The awake state characterized by dominant alpha activity in the electroencephalography signals.
  \item \textbf{N1}: The transitional stage from wakefulness to sleep, exhibiting low-amplitude mixed-frequency activity.
  \item \textbf{N2}: A light sleep stage identified by the presence of sleep spindles and K-complexes.
  \item \textbf{N3}: The deep sleep stage dominated by high-amplitude slow-wave activity.
  \item \textbf{REM}: Rapid eye movement sleep, associated with vivid dreaming and rapid eye movements.
\end{itemize}
In this experiment, 10 channels are used, including 6 electroencephalography (EEG) channels, namely C3-A2, C4-A1, F3-A2, F4-A1, O1-A2, and O2-A1, 2 electrooculography (EOG) channels, namely LOC-A2 and ROC-A1, 1 chin electromyography (EMG) channel, and 1 electrocardiography (ECG) channel. All signals are downsampled from 200 Hz to 100 Hz. The final 30 seconds sleep of each epoch are excluded. To exploit the neighborhood information among adjacent segments, each target segment is concatenated with its two neighboring segments and used jointly as the input for subsequent analysis. The number of samples for each sleep stage is summarized as Table~\ref{tab:isruc_stage_distribution}.

\begin{table}[htbp]
\centering
\caption{The number of samples for each sleep stage of the ISRUC-S3 dataset.}
\label{tab:isruc_stage_distribution}
\setlength{\tabcolsep}{10pt}
\begin{tabular}{c ccccc c}
\toprule
Subject & Wake & N1 & N2 & N3 & REM & Total \\
\midrule
1  & 162 & 114 & 369 & 178 & 101 & 924 \\
2  & 100 & 141 & 323 & 197 & 150 & 911  \\
3  & 80  & 64  & 249 & 298 & 103 & 794  \\
4  & 144 & 137 & 234 & 159 & 90  & 764  \\
5  & 301 & 70  & 267 & 195 & 81  & 914  \\
6  & 57  & 133 & 290 & 247 & 96  & 823  \\
7  & 206 & 63  & 155 & 262 & 98  & 784  \\
8  & 375 & 109 & 194 & 143 & 149 & 970  \\
9  & 122 & 167 & 362 & 225 & 63  & 939  \\
10 & 127 & 219 & 173 & 112 & 135 & 766  \\
\midrule
\textbf{Total} & \textbf{1674} & \textbf{1217} & \textbf{2616} & \textbf{2016} & \textbf{1066} & \textbf{8589} \\
\bottomrule
\end{tabular}
\end{table}
\subsection{Baselines}\label{Baseline methods}
To comprehensively evaluate the effectiveness of CDM in the sleep staging task on the ISRUC-S3 dataset, we select several representative benchmark methods for both clustering and classification tasks.
\subsubsection{Baselines for clustering}\label{baselines_clustering}
\begin{itemize}
    \item \textbf{PCA}~\cite{wold1987principal}: 
    A classical linear dimensionality reduction method that projects high-dimensional data onto a low-dimensional subspace by maximizing the variance of the projected samples while preserving the dominant global structure.

    \item \textbf{KPCA}~\cite{scholkopf1997kernel}: 
    A nonlinear extension of principal component analysis that maps data into a high-dimensional feature space through kernel functions and performs dimensionality reduction by extracting principal components in the transformed space.

    \item \textbf{ISOMAP}~\cite{tennenbaum2000global}: 
    A manifold learning method that preserves the global geometric structure of data by approximating geodesic distances on a neighborhood graph and embedding the samples into a low-dimensional space.

    \item \textbf{LLE}~\cite{roweis2000nonlinear}: 
    A neighborhood-preserving manifold learning approach that computes low-dimensional embeddings by maintaining the local linear reconstruction relationships among neighboring data points.

    \item \textbf{DM}~\cite{coifman2005geometric}: 
    A diffusion-based manifold learning method that constructs a Markov random walk on the data graph and derives low-dimensional representations by capturing the intrinsic geometry and connectivity structure of the data.
\end{itemize}
\subsubsection{Baselines for classification}\label{baselines_classify}
\begin{itemize}
    \item \textbf{RF}~\cite{memar2017novel}: 
    A feature-based sleep stage classification method that decomposes EEG signals into multiple frequency subbands, applies statistical and relevance-based feature selection, and performs classification using a random forest ensemble.
    \item \textbf{SVM}~\cite{alickovic2018ensemble}: 
    A single-channel EEG sleep scoring framework that combines wavelet-based feature extraction with a rotational support vector machine to enhance robustness and classification accuracy.
    \item \textbf{MLP+LSTM}~\cite{dong2017mixed}: 
    A EEG sleep stage classification framework that combines hierarchical feature extraction using multilayer perceptrons with long short-term memory networks to model temporal dependencies in sleep stage transitions.
    \item \textbf{SalientSleepNet}~\cite{ijcai2021p360}: 
    A multimodal deep learning model that leverages parallel encoder--decoder convolutional pathways to hierarchically capture salient temporal patterns from heterogeneous physiological signals.
    \item \textbf{MSTGCN}~\cite{jia2021multi}: 
    A multi-view spatial--temporal graph convolutional network that jointly models brain connectivity, temporal transitions, and domain generalization to learn subject-invariant sleep representations.
    \item \textbf{AttenNet}~\cite{eldele2021attention}: 
    A deep neural architecture that integrates convolutional feature extraction with multi-head attention and causal temporal modeling to emphasize discriminative sleep-related dynamics.
    \item \textbf{SleepPrintNet}~\cite{jia2021sleepprintnet}: 
    A multimodal network that fuses complementary temporal, spectral--spatial, and cross-modal EEG features to form discriminative sleep stage representations.
\end{itemize}

\subsection{Evaluation matrices}\label{Evaluation matrices}
Let the dataset contain $N$ = 8589 samples, the true labels be denoted as $\bm{Y}=\{y_1,y_2,\cdots,y_N\}$, and the prediction results be denoted as $\hat{\bm{Y}}=\{\hat{y}_1,\hat{y}_2,\dots,\hat{y}_N\}$. In unsupervised clustering, $\hat{y}_i$ represents the cluster label of the $i$-th sample. In supervised classification, $\hat{y}_i$ represents the predicted class label of the $i$-th sample. Let the total number of classes be $K=5$. 

\subsubsection{Evaluation matrices for clustering}\label{Evaluation_matrices_clu}
\begin{itemize}
    \item \textbf{Clustering Accuracy, (${\mathrm{ACC}_{\mathrm{clu}}}$):}  
    ${\mathrm{ACC}_{\mathrm{clu}}}$ evaluates clustering performance by measuring the maximum matching accuracy between clustering assignments and ground-truth labels after optimal label permutation. It is defined as
    \[
        \mathrm{ACC}_{\mathrm{clu}}
        =
        \frac{1}{N}\max_{\pi}\sum_{i=1}^{N}\mathbf{1}\!\left(y_i=\pi(\hat{y}_i)\right),
    \]
    where $\pi(\cdot)$ denotes the permutation mapping from clustering labels to true labels, and $\mathbf{1}(\cdot)$ is the indicator function, which equals $1$ if the condition holds and $0$ otherwise. $\mathrm{ACC}_{\mathrm{clu}}\in[0,1]$, and a larger value indicates better agreement between clustering results and ground-truth labels.

    \item \textbf{Adjusted Rand Index(ARI)~\cite{hubert1985comparing}:}  
    ARI is a pairwise clustering similarity metric that quantifies the agreement between clustering assignments and ground-truth labels while correcting for chance. For any pair of samples, let:
    \begin{itemize}
    \renewcommand{\labelitemii}{$\scriptscriptstyle\bullet$}
        \item $\mathrm{TP}$: the number of sample pairs assigned to the same class in both the ground-truth labels and clustering results;
        \item $\mathrm{TN}$: the number of sample pairs assigned to different classes in both the ground-truth labels and clustering results;
        \item $\mathrm{FP}$: the number of sample pairs assigned to the same cluster but belonging to different ground-truth classes;
        \item $\mathrm{FN}$: the number of sample pairs assigned to different clusters but belonging to the same ground-truth class.
    \end{itemize}
    The Rand Index (RI) is first defined as
    \[
        \mathrm{RI}
        =
        \frac{\mathrm{TP}+\mathrm{TN}}
        {\mathrm{TP}+\mathrm{TN}+\mathrm{FP}+\mathrm{FN}}.
    \]
    Based on RI, the Adjusted Rand Index is given by
    \[
        \mathrm{ARI}
        =
        \frac{\mathrm{RI}-\mathbb{E}(\mathrm{RI})}
        {\max(\mathrm{RI})-\mathbb{E}(\mathrm{RI})},
    \]
    where $\mathbb{E}(\mathrm{RI})$ denotes the expected RI under random label assignments, and $\max(\mathrm{RI})$ denotes the maximum possible value of RI.  
    $\mathrm{ARI}\in[-1,1]$. A value of $1$ indicates perfect agreement, a value close to $0$ indicates random-level agreement, and a negative value indicates worse-than-random clustering performance.

    \item \textbf{Normalized Mutual Information,(NMI)~\cite{vinh2009information}:}  
    NMI is an information-theoretic metric that evaluates clustering quality by measuring the statistical dependence between clustering assignments and true labels. It is defined as
    \[
        \mathrm{NMI}
        =
        \frac{I(Y,\hat{Y})}{\sqrt{H(Y)H(\hat{Y})}},
    \]
    where $I(Y,\hat{Y})$ denotes the mutual information between the ground-truth label variable $Y$ and the clustering label variable $\hat{Y}$, and $H(Y)$ and $H(\hat{Y})$ denote their corresponding entropies. Specifically,
    \[
        I(Y,\hat{Y})
        =
        \sum_{i=1}^{K}\sum_{j=1}^{K}
        p(i,j)\log\frac{p(i,j)}{p(i)p(j)},
    \]
    \[
        H(Y)=-\sum_{i=1}^{K}p(i)\log p(i), \qquad
        H(\hat{Y})=-\sum_{j=1}^{K}p(j)\log p(j),
    \]
    where $p(i)$ is the probability that a sample belongs to the true class $i$, $p(j)$ is the probability that a sample is assigned to cluster $j$, and $p(i,j)$ is their joint probability.  
    $\mathrm{NMI}\in[0,1]$, and a larger value indicates stronger consistency between clustering assignments and true labels.

    \item \textbf{Fisher Discriminant Ratio, (FDR)~\cite{fisher1936use}:}  
    FDR evaluates the discriminative quality of low-dimensional embeddings by comparing between-class scatter with within-class scatter. Let $N_k$ denote the number of samples in class $k$, $\boldsymbol{\mu}_k$ the mean embedding vector of class $k$, and $\boldsymbol{\mu}$ the global mean embedding vector of all samples. The between-class scatter matrix and within-class scatter matrix are defined as
    \[
        \boldsymbol{S}_B
        =
        \sum_{k=1}^{K} N_k (\boldsymbol{\mu}_k-\boldsymbol{\mu})(\boldsymbol{\mu}_k-\boldsymbol{\mu})^{\top},
    \]
    \[
        \boldsymbol{S}_W
        =
        \sum_{k=1}^{K}\sum_{y_i=k}(\boldsymbol{f}_i-\boldsymbol{\mu}_k)(\boldsymbol{f}_i-\boldsymbol{\mu}_k)^{\top},
    \]
    where $\boldsymbol{f}_i$ denotes the embedding vector of the $i$-th sample. Based on these quantities, the Fisher Discriminant Ratio is defined as
    \[
        \mathrm{FDR}
        =
        \frac{\operatorname{tr}(\boldsymbol{S}_B)}{\operatorname{tr}(\boldsymbol{S}_W)},
    \]
    where $\operatorname{tr}(\cdot)$ denotes the trace of a matrix. A larger $\mathrm{FDR}$ indicates better class separability, corresponding to larger inter-class distances and smaller intra-class variations.
\end{itemize}

\subsubsection{Evaluation matrices for classification}\label{Evaluation_matrices_cls}
\begin{itemize}
    \item \textbf{Accuracy ($\mathrm{ACC_{cls}}$)}:
    $\mathrm{ACC_{cls}}$ measures the proportion of correctly classified sleep epochs among all samples. Let $\boldsymbol{C} \in \mathbb{R}^{K \times K}$ denote the confusion matrix, where $\boldsymbol{C}_{ii}$ represents the number of correctly classified samples of class $i$.
    The overall accuracy is defined as
    \begin{displaymath}
        \mathrm{ACC_{cls}} = \frac{\sum_{i=1}^{K} \boldsymbol{C}_{ii}}{\sum_{i=1}^{K} \sum_{j=1}^{K} \boldsymbol{C}_{ij}}.
    \end{displaymath}
    $\mathrm{ACC}_{\mathrm{cls}}\in[0,1]$. When $\mathrm{ACC}_{\mathrm{cls}}$ is closer to $1$, it indicates a higher consistency between the predicted labels and the true labels, and better classification performance.

    \item \textbf{F1-score}:
    F1-score is the harmonic mean of Precision and Recall, and is used to comprehensively evaluate the balance between the model's positive predictive ability and sensitivity. For the $i$-th class, let
    \begin{itemize}
    \renewcommand{\labelitemii}{$\scriptscriptstyle\bullet$}
        \item $\mathrm{TP}_i$: the number of samples that are correctly predicted as belonging to the $i$-th class;
        \item $\mathrm{FP}_i$: the number of samples that are predicted as belonging to the $i$-th class but actually belong to other classes;
        \item $\mathrm{FN}_i$: the number of samples that actually belong to the $i$-th class but are predicted as other classes.
    \end{itemize}
    
    These quantities can be computed from the confusion matrix $\boldsymbol{C}$ as
    \[
        \mathrm{TP}_i=C_{ii},\qquad
        \mathrm{FP}_i=\sum_{j=1}^{K}C_{ji}-C_{ii},\qquad
        \mathrm{FN}_i=\sum_{j=1}^{K}C_{ij}-C_{ii}.
    \]
    
    Accordingly, the Precision and Recall for the $i$-th class are defined as
    \[
        \mathrm{Precision}_i
        =
        \frac{\mathrm{TP}_i}{\mathrm{TP}_i+\mathrm{FP}_i},
        \qquad
        \mathrm{Recall}_i
        =
        \frac{\mathrm{TP}_i}{\mathrm{TP}_i+\mathrm{FN}_i}.
    \]
    
    The F1-score for the $i$-th class is then given by
    \[
        \mathrm{F1}_i
        =
        \frac{2\,\mathrm{Precision}_i\,\mathrm{Recall}_i}
        {\mathrm{Precision}_i+\mathrm{Recall}_i}.
    \]
    
    The overall F1-score across all classes is computed using the macro-average:
    \[
        \mathrm{F1}
        =
        \frac{1}{K}\sum_{i=1}^{K}\mathrm{F1}_i.
    \]
    
    $\mathrm{F1}\in[0,1]$. A larger value indicates better classification performance. As a macro-averaged metric, F1 assigns equal importance to each class and can therefore provide a more comprehensive evaluation of the model's classification performance across different sleep stages.

    \item \textbf{Cohen's Kappa ($\kappa$)}:
    The Kappa coefficient $\kappa$ measures the degree of agreement between model predictions and true labels beyond chance-level consistency. 
    
    The observed agreement $p_o$, which coincides with the definition of accuracy above, is given by
    \[
        p_o = \frac{\sum_{i=1}^{K} \boldsymbol{C}_{ii}}{\sum_{i=1}^{K} \sum_{j=1}^{K} \boldsymbol{C}_{ij}}.
    \]
    The expected agreement by chance is defined as
    \[
        p_e = \frac{1}{N^2} \sum_{i=1}^{K}
        \left( \sum_{j=1}^{K} \boldsymbol{C}_{ij} \right)
        \left( \sum_{j=1}^{K} \boldsymbol{C}_{ji} \right).
    \]
    Accordingly, Cohen's Kappa coefficient is defined as
    \[
        \kappa = \frac{p_o - p_e}{1 - p_e}.
    \]
    $\kappa \in [0,1]$. It provides a more objective assessment of classification reliability, particularly in scenarios with class imbalance.
\end{itemize}

\subsection{Results of cross-subject classification}\label{appendix_eeg_cross_result}
This section presents the results of the \textbf{cross-subject setting} on the ISRUC-S3 dataset~\cite{khalighi2016isruc}. Cross-subject means that data from 8 subjects are used for training, data from 1 subject for validation, and data from the remaining subjects for testing. Each subject is used only once as the test set, achieving 10-fold cross-validation. This setup is designed to evaluate CDM's cross-subject generalization ability.

\begin{figure}[htbp]
    \centering
    \includegraphics[width=0.95\textwidth]{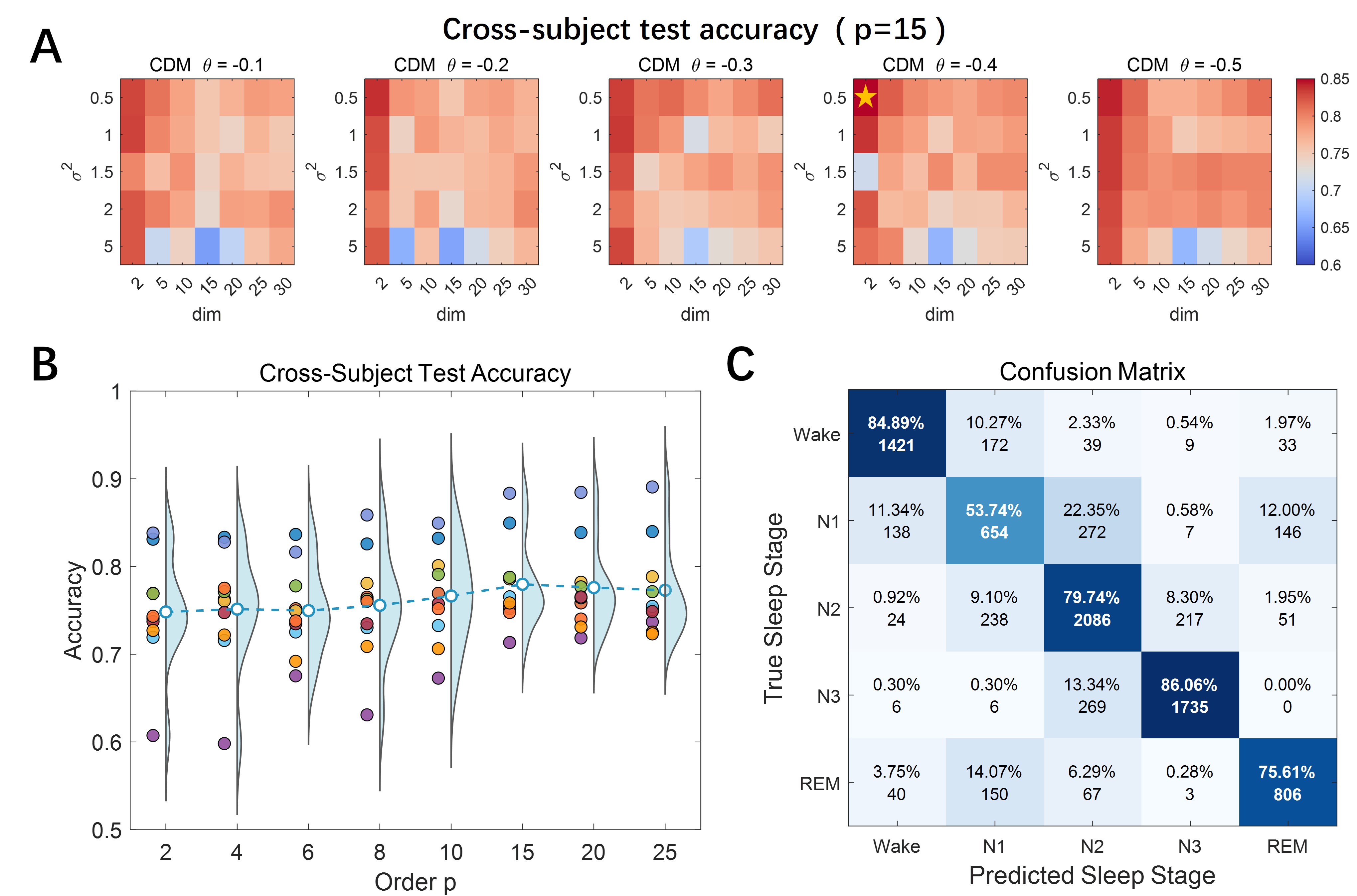}
    \caption{\footnotesize
10-fold cross-validation results of CDM under the cross-subject setting on the ISRUC-S3 dataset~\cite{khalighi2016isruc}. (A) Test accuracy $\mathrm{ACC_{cls}}$ of one fold under the stacking order $p=15$. The variable \textit{dim} denotes the dimension of the low-dimensional embeddings used for SVM~\cite{cortes1995support} classification. The best performance is marked with a star. (B) The best test accuracy of each fold under different order $p$, where colored scatter points represent individual folds. The hollow marker indicates the mean accuracy over the 10 folds, and the optimal stacking order is $p=15$. (C) The confusion matrix corresponding to the overall results at the optimal $p=15$.
}
\label{eeg_class_cross_result}
\end{figure}

For the classification baselines, several deep learning models specifically designed for EEG sleep stage classification are selected for comparison, as listed in~\ref{baselines_classify}. To ensure fairness, each method use its original input format without additional high-order stacking. Hyperparameters, including the number of layers, learning rate, and batch size, are selected by grid search to obtain the best-performing configuration under each experimental setting.

Fig.~\ref{eeg_class_cross_result} presents the classification performance of CDM under the cross-subject settings. The results show that appropriate order stacking remains effective for processing larger sample sizes. Fig.~\ref{eeg_class_cross_result}~(B) indicates that, in the cross-subject case, the optimal order of CDM is $p=15$. Comparing the optimal order between the intra-subject (see Fig.~\ref{eeg_class_result}) and the cross-subject setting, the former achieves a relatively lower order ($p=4$) than the latter ($p=15$). This suggests that low-order representations are sufficient to model dynamics from a single subject , where subject-specific characteristics are relatively consistent. In contrast, cross subject modeling requires substantially richer feature hierarchies and deeper spatial stacking to capture higher-dimensional latent maps that account for inter-subject variability and distribution shifts across subjects. 

\begin{table*}[htbp]
\centering
\small
\setlength{\tabcolsep}{4pt}
\caption{\footnotesize{Comparison of the sleep stage classification performance under the cross-subject setting. Evaluation metrics are defined in Appendix~\ref{Evaluation_matrices_cls}. The bolded and underlined results represent the best and second-best, respectively, for partial results.}}
\label{tab:isruc_s3_cross_comparison}
\setlength{\tabcolsep}{6pt}
\begin{tabular}{c c c c c c c c c c}
\toprule
\multirow{2}{*}{\makecell{Setting}} &
\multirow{2}{*}{Method} &
\multicolumn{3}{c}{Overall results} &
\multicolumn{5}{c}{F1-score for each class}\\
\cmidrule(lr){3-5} \cmidrule(lr){6-10}
& & $\mathrm{ACC_{cls}}$ & F1 & Kappa & Wake & N1 & N2 & N3 & REM  \\
\cmidrule(lr){1-10}
\multirow{8}{*}{\makecell{Cross-\\subject}}
& RF~\cite{memar2017novel}                
& 0.729 & 0.708 & 0.648 & 0.858 & 0.473 & 0.704 & 0.809 & 0.699 \\
& SVM~\cite{alickovic2018ensemble}        
& 0.733 & 0.721 & 0.657 & \textbf{0.868} & \underline{0.523} & 0.699 & 0.786 & 0.731 \\
& MLP+LSTM~\cite{dong2017mixed}           
& \underline{0.779} & \underline{0.758} & \underline{0.713} & 0.860 & 0.469 & \underline{0.760} & \textbf{0.875} & \textbf{0.828}\\
& \textbf{CDM}                    
& \textbf{0.780} & \textbf{0.763} & \textbf{0.717} & \underline{0.860} & \textbf{0.537} & \textbf{0.780} & \underline{0.870} & \underline{0.767} \\
\cmidrule(lr){2-10}
& SalientSleepNet~\cite{ijcai2021p360}    
& 0.785 & 0.764 & 0.723 & 0.837 & 0.521 & 0.795 & 0.875 & 0.793\\
& MSTGCN~\cite{jia2021multi}              
& 0.799 & 0.788 & 0.740 & 0.872 & 0.567 & 0.789 & 0.874 & 0.837\\
& AttenNet~\cite{eldele2021attention}    
& 0.801 & 0.776 & 0.743 & 0.893 & 0.540 & 0.809 & 0.889 & 0.748\\
& SleepPrintNet~\cite{jia2021sleepprintnet} 
& 0.825 & 0.806 & 0.775 & 0.894 & 0.592 & 0.852 & 0.895 & 0.829\\
\bottomrule
\end{tabular}
\end{table*}

For a comprehensive evaluation, we still compared CDM with several baselines. All baseline methods are specially designed for EEG-based sleep signals. A brief description of each method is in Appendix ~\ref{Baseline methods}. Table~\ref{tab:isruc_s3_cross_comparison} reports the classification performance on the ISRUC-S3 dataset under the cross-subject settings. It shows that CDM outperforms conventional machine learning approaches (RF and SVM) as well as the hybrid MLP+LSTM model in terms of the classification accuracy $\mathrm{ACC_{cls}}$, average F1 score, and the Kappa coefficient. 

The results in the upper part of Table~\ref{tab:isruc_s3_cross_comparison} also shows that CDM achieves higher F1-scores than the other three methods for both the N1 and N2 stages. N1 and N2 are transitional sleep stages that are inherently similar and therefore more prone to confusion in classification tasks. One possible reason for this advantage is that CDM can extract features beyond Euclidean distance, enabling it to capture more complex relationships embedded in the data. This allows it to better capture subtle differences in easily confused signals such as N1 and N2. This demonstrates that CDM does indeed have a stronger ability to discriminate confused signals in real-world, complex scenarios.

\begin{figure}[htbp]
    \centering
    \includegraphics[width=0.85\textwidth]{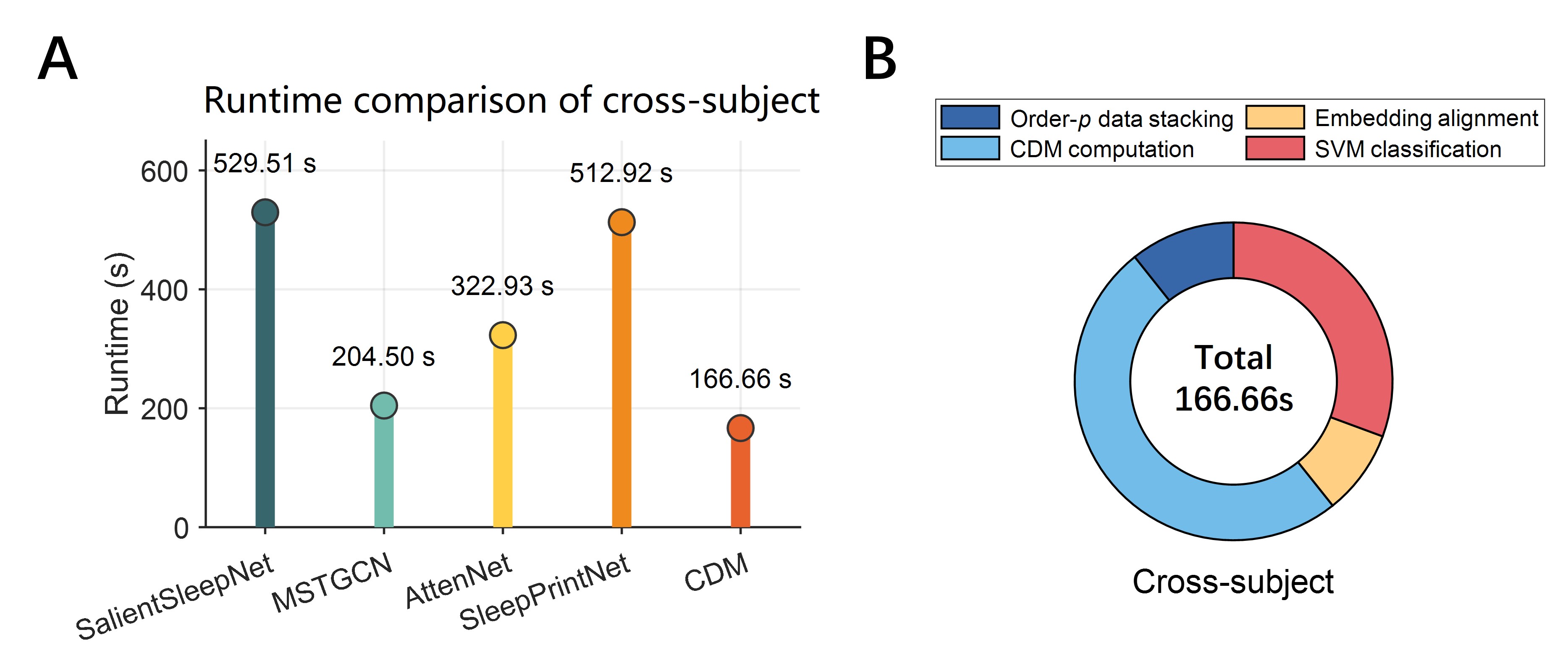}
    \caption{\footnotesize{Display of runtime on cross-subject setting. Runtime refers to the average runtime of 10-fold experiments with their respective optimal parameters. (A) Comparison of the runtime of CDM with baselines. (B) Breakdown of CDM’s runtime across its four phases, Order-$p$ data stacking, CDM computation, embedding alignment, and SVM classification.} } 
    \label{runtime_cross}
\end{figure}

Combined with the lower part of Table~\ref{tab:isruc_s3_cross_comparison} and Fig.~\ref{runtime_cross}, it can be observed that CDM maintains a clear advantage in computational efficiency compared with neural network models specifically tailored for EEG-based sleep stage classification. In particular, CDM reduces the runtime by approximately $67.5\%$ compared with the best-performing deep learning baseline, highlighting its superiority in terms of computational cost and resource efficiency.

\begin{figure}[htbp]
    \centering
    \includegraphics[width=0.55\textwidth]{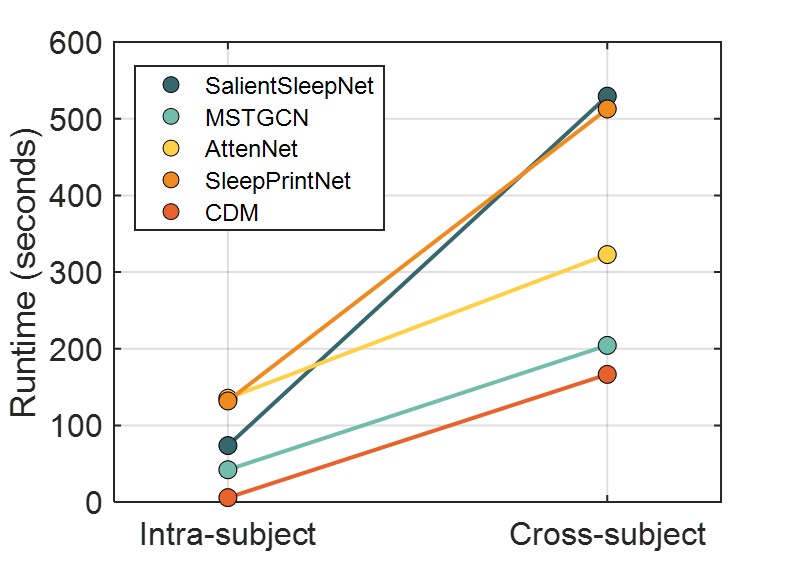}
    \caption{\footnotesize{Display of runtime on intra-subject and cross-subject. }} 
    \label{runtime_intra_cross}
\end{figure}

Fig.~\ref{runtime_intra_cross} highlights the scalability of the proposed CDM. Although the runtime of all methods increases with larger training sets, multimodal neural network–based approaches suffer from a much steeper growth in computational cost, whereas CDM remains consistently efficient with the lowest overall runtime and a moderate scaling trend. A detailed runtime breakdown further indicates that the core CDM computation accounts for only about half of the total cost, and the dominant increase in runtime with larger samples mainly arises from the classifier stage, demonstrating the computational efficiency and scalability of the proposed framework.

\bibliographystyle{unsrt}
\newpage
\bibliography{reference}

\end{document}